\documentclass[10pt]{article}

\usepackage[margin=0.70in,twocolumn]{geometry}
\usepackage{times}

\usepackage{microtype}
\usepackage{graphicx}
\usepackage{float}
\usepackage{balance}
\graphicspath{{./}}
\usepackage{booktabs}
\usepackage{amsmath,amssymb,amsthm}
\usepackage{mathtools}
\usepackage{enumitem}
\usepackage{xcolor}
\usepackage[hidelinks]{hyperref}
\usepackage{xurl}
\usepackage[super,sort&compress]{natbib}
\usepackage{tikz}
\usetikzlibrary{arrows.meta,positioning,shapes.geometric,calc}
\newenvironment{boxedtext}{%
  \begin{center}
  \setlength{\fboxsep}{8pt}
  \begin{tabular}{|p{0.94\linewidth}|}
  \hline
  \small
}{%
  \\
  \hline
  \end{tabular}
  \end{center}
}

\newcommand{\repoURL}{https://github.com/SR123/LLM-text-ecosystems}

\newtheorem{theorem}{Theorem}
\newtheorem{proposition}{Proposition}
\newtheorem{lemma}{Lemma}
\newtheorem{corollary}{Corollary}
\theoremstyle{definition}
\newtheorem{definition}{Definition}
\newtheorem{example}{Example}
\theoremstyle{remark}
\newtheorem{remark}{Remark}

\newcommand{\1}{\mathbf{1}}
\newcommand{\Prob}{\mathbb{P}}
\newcommand{\E}{\mathbb{E}}
\newcommand{\KL}{\mathrm{KL}}

\definecolor{tutorialblue}{RGB}{0,51,102}

\title{\bfseries Drift and selection in LLM text ecosystems}
\author{S\o ren Riis\\
\small Queen Mary University of London, London, United Kingdom\\
\small \texttt{s.riis@qmul.ac.uk}}
\date{}

\begin{document}

\twocolumn[
\maketitle
\vspace{-1.05em}
\begin{center}
\begin{minipage}{0.97\textwidth}\small
\textbf{Abstract.}
The public text record---the material from which both people and AI systems now learn---is increasingly shaped by its own outputs.
Generated text enters the public record, later agents learn from it, and the cycle repeats.
Here we develop an exactly solvable mathematical framework for this recursive process, based on variable-order $n$-gram agents, and separate two forces acting on the public corpus.
The first is drift: unfiltered reuse progressively removes rare forms, and in the infinite-corpus limit we characterise the stable distributions exactly.
The second is selection: publication, ranking and verification filter what enters the record, and the outcome depends on what is selected.
When publication merely reflects the statistical status quo, the corpus converges to a shallow state in which further lookahead brings no benefit.
When publication is normative---rewarding quality, correctness or novelty---deeper structure persists, and we establish an optimal upper bound on the resulting divergence from shallow equilibria.
The framework therefore identifies when recursive publication compresses public text and when selective filtering sustains richer structure, with implications for the design of AI training corpora.
\end{minipage}
\end{center}
\vspace{0.45em}
]

\section*{A human--AI feedback loop in public text}

Public text is increasingly produced and filtered by mixed human--AI systems: models generate drafts and suggestions, humans choose what to publish, and automated ranking, moderation, testing, verification and corpus-curation pipelines determine which outputs are surfaced, retained, deduplicated\cite{Dodge2021C4} or suppressed. Some of this material later becomes training data for new models,\cite{RefinedWeb2023,FineWeb2024,DataCompLM2024,Albalak2024Survey} so later learners train not on an untouched web but on a public record already altered by earlier generators and filters.

Prior work has shown that recursive reuse of synthetic data can delete tails, reduce diversity or induce collapse-like behaviour.\cite{Shumailov2024Nature,Alemohammad2024MAD,BohacekFarid2023,Guo2024Diversity} Related evidence from computer vision shows analogous contamination effects when generated outputs re-enter future datasets.\cite{Hataya2023Corrupt} Other work shows that the outcome depends strongly on the surrounding pipeline: iterative retraining can remain stable in some regimes,\cite{Bertrand2024} accumulation of real data can break the recursive curse,\cite{Gerstgrasser2024} and synthetic-data loops can behave very differently when verification or correction is built into the loop.\cite{Feng2024Verification,Gillman2024SelfCorrecting} These results have mostly been studied in isolation. We therefore do \emph{not} claim to discover model collapse per se. Our contribution is a unified theory of \emph{public text environments} that separates neutral drift from selective publication and studies what later learners inherit from the resulting public record.

We introduce an exactly solvable framework for this recursive process based on variable-order $n$-gram agents.\cite{Katz1987,KneserNey1995,Shannon1951,GriffithsKalish2007,KirbyGriffithsSmith2014,BoydRicherson1985} The framework is mathematically transparent: the relevant conditional distributions can be written down explicitly and the long-run fixed points can be characterised exactly. It separates two forces. Drift arises from repeated reuse of finite corpora and removes weakly supported forms even without any preference over content. Selection arises because publication, ranking, verification and omission change what becomes visible in the public record.

In this framework, unfiltered reuse drives the public record toward shallow equilibria, in which the visible statistics can be matched without recovering deeper generative structure. By contrast, normative selection---publication rules that reward quality, correctness or novelty---can preserve richer structure. Recursive publication therefore has no single universal effect: it can stabilise polished outputs while removing the intermediate traces needed for learning from process rather than outcome alone.

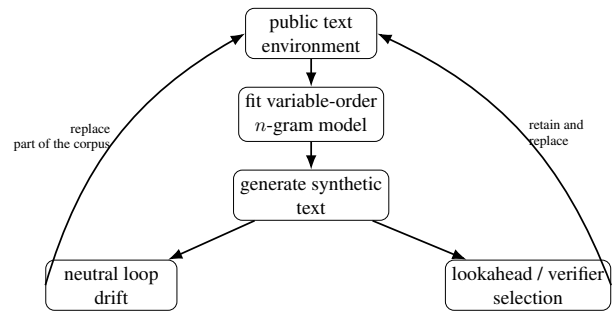
\begin{figure}[t]
\centering
\resizebox{0.92\columnwidth}{!}{%
\begin{tikzpicture}[node distance=5mm,>=Latex,scale=0.9, every node/.style={transform shape}]
\node[draw,rounded corners,align=center,minimum width=2.35cm,minimum height=0.72cm] (env) {public text\\environment};
\node[draw,rounded corners,align=center,minimum width=2.35cm,minimum height=0.72cm,below=of env] (train) {fit variable-order\\$n$-gram model};
\node[draw,rounded corners,align=center,minimum width=2.35cm,minimum height=0.72cm,below=of train] (gen) {generate synthetic\\text};
\node[draw,rounded corners,align=center,minimum width=2.35cm,minimum height=0.72cm,below left=7mm and 10mm of gen] (drift) {neutral loop\\drift};
\node[draw,rounded corners,align=center,minimum width=2.35cm,minimum height=0.72cm,below right=7mm and 10mm of gen] (sel) {lookahead / verifier\\selection};
\draw[-Latex, thick] (env) -- (train);
\draw[-Latex, thick] (train) -- (gen);
\draw[-Latex, thick] (gen) -- (drift);
\draw[-Latex, thick] (gen) -- (sel);
\draw[-Latex, thick] (drift.west) to[bend left=22] node[left, font=\scriptsize, align=right]{replace\\part of the corpus} (env.west);
\draw[-Latex, thick] (sel.east) to[bend right=22] node[right, font=\scriptsize, align=left]{retain and\\replace} (env.east);
\end{tikzpicture}%
}
\caption{\textbf{Recursive text as drift plus selection.}
A finite sample of the environment is used to fit a short-context generator.
The generator produces synthetic text.
What re-enters the environment can be unfiltered (drift) or filtered by a success criterion (selection).}
\label{fig:loop}
\end{figure}

\section*{Neutral recursion is Wright--Fisher drift}

The neutral baseline is simple: fit an $n$-gram model to the current corpus, retain a fraction of the documents, regenerate the rest from the fitted model, and repeat.
In population-genetic terms, this is a Wright--Fisher-style reproduction process on text: rare forms are not selected against, but finite sampling makes them vulnerable to loss.

When the corpus is very large---thought of, as in real-world training, as a large collection of independent documents rather than a single long text---the fitted $n$-gram model converges to a well-defined conditional distribution determined by the current corpus.
The stochastic instability is therefore a finite-sample phenomenon.
If a word has frequency $p$ and the regenerated batch contains $M$ sampled tokens, the chance that it is absent from that batch is
\begin{equation}
\Pr(\text{word absent from batch})=(1-p)^M.
\label{eq:dropout}
\end{equation}
Rare forms are therefore the first to disappear.
Without smoothing (that is, in any model that assigns \emph{zero} probability to unseen events), tail support is systematically fragile.

Theorem~1 captures the two baseline facts for the unsmoothed recursion: in finite corpora the unigram case is exactly Wright--Fisher drift, while in the infinite-corpus limit the general $n$-gram recursion admits a complete fixed-point description.

\begin{boxedtext}
\textbf{Theorem 1 (drift and the fixed-point polytope).}
\emph{Part (a): finite-sample drift.}
Consider the unigram urtext recursion in which, at each generation, a fraction $\alpha\in[0,1]$ of the public corpus is replaced by a fresh synthetic batch of $M$ tokens drawn i.i.d.\ from the current model.
Let $\mu_t$ be the mass of any chosen minority vocabulary at generation $t$.
Then
\[
\mu_{t+1}=(1-\alpha)\mu_t+\alpha K_t/M,
\qquad K_t\sim \mathrm{Bin}(M,\mu_t).
\]
Consequently
\[
\mathbb E[\mu_{t+1}\mid \mu_t]=\mu_t,
\qquad
\mathrm{Var}(\mu_{t+1}\mid \mu_t)=\alpha^2\mu_t(1-\mu_t)/M,
\]
and the one-step dropout probability (zero count in the synthetic batch) is $(1-\mu_t)^M$.
When $\alpha=1$ this is exactly the classical Wright--Fisher drift process from population genetics.
Rare forms behave like rare alleles in a finite population: their expected mass is unchanged from one generation to the next, but finite-sample fluctuations make the rarest forms the most likely to be lost.

\emph{Part (b): complete characterisation of fixed points in the infinite-corpus limit.}
In the limit $M\to\infty$, sampling noise vanishes and the distributional recursion $\rho_{t+1} = G_n(R_n(\rho_t))$ becomes exact: the updated $n$-gram law is fully determined by the current one.
The set of all its fixed points is a convex polytope---the polytope of non-negative unit circulations on the \emph{de~Bruijn graph} $B(n{-}1, s)$, whose nodes are $(n{-}1)$-grams and whose edges are $n$-grams.
Its dimension is $s^{n-1}(s-1)$.
The extreme points of this polytope are in bijection with simple directed cycles in $B(n{-}1, s)$: each extreme point is the uniform distribution on the $n$-grams traversed by a single deterministic periodic sequence.
Every self-consistent $n$-gram distribution is a convex combination of these deterministic extremes.
\end{boxedtext}

Equation~\eqref{eq:dropout} and Theorem~1(a) show that the expected frequency of a minority form is unchanged from one generation to the next: $\mathbb E[\mu_{t+1}\mid \mu_t]=\mu_t$, so there is no systematic force driving any word toward extinction.
What does change is the variance around that expectation, which accumulates under finite sampling: each generation introduces fresh random fluctuations, and in the pure-replacement case ($\alpha=1$), a rare form sampled zero times is lost permanently.

Theorem~1(b) gives a complete description of the deterministic attractors.
For a binary alphabet, trigrams yield a $4$-dimensional polytope with $6$~extreme points; $4$-grams yield an $8$-dimensional polytope with $19$~extreme points.
The number of extreme points grows super-exponentially with both alphabet size and context length: for an alphabet of size~$10$ in the bigram case it already exceeds one million, and for realistic vocabularies it is astronomically large.\cite{Maurer1992}
Every extreme point corresponds to a simple cycle in the de~Bruijn graph---the uniform distribution over the $n$-grams that the cycle traverses---and every fixed point is a convex combination of these cycle distributions, parameterised by continuously many mixture weights.
Although every extreme point is periodic, a generic fixed point is a mixture of cycles with different periods and is itself non-periodic for any practical purpose; the resulting text can exhibit irregular patterns of considerable length.
Worked examples, verification code, and the general circulation-polytope proof are in the supplementary materials.

Theorem~1 describes the pure unsmoothed baseline.
In practice, back-off and smoothing change how loss manifests at higher orders.
For longer contexts, the same finite-sample pressure acts on phrases and on the context windows that support them.
What disappears first is often not the word itself but the highest-order support that allows the phrase to be regenerated.
Back-off keeps the generator fluent, but by shortening memory it turns distinctive prose into more generic prose.
Common forms therefore enjoy a compounding advantage over rare forms.

\section*{Lookahead turns next-token prediction into selection}

Standard next-token prediction evaluates tokens by their immediate conditional likelihood. But a token that looks locally likely may still lead to a continuation that is later rejected. A natural alternative is therefore to ask not which token looks best \emph{now}, but which token is most likely to remain acceptable for a prescribed number of further steps.

A \emph{trace} is the emitted token sequence. Whether a trace succeeds or fails may become clear only after several further steps, as later tokens reveal whether the continuation remains viable. The effective public next-token law is then the generator's law conditioned on the continuation surviving. This is not a claim that a modern LLM literally enumerates every $L$-step future. The same reweighting can arise from chain-of-thought reasoning, self-consistency selection, or external verifiers that suppress failing branches.\cite{Wei2022CoT,Wang2023SelfConsistency,Lightman2023,Feng2024Verification,Gillman2024SelfCorrecting} What reaches the public record is typically only the accepted branch; the hidden search remains off-stage.

This induces selection. A token is favoured not because it is locally common, but because it leads to continuations that survive the acceptance process. In code, proof and formal reasoning, this filter can be grounded in explicit success predicates. In open-ended prose, it may instead reinforce patterns already favoured by existing social or editorial preferences.

\section*{Stable public equilibria under lookahead selection}

The key theoretical question is whether repeated lookahead-guided publication can settle to a stable distribution, and if so, what that distribution looks like.

Write the \emph{corpus $r$-gram distribution} for the distribution of contiguous $r$-token blocks measured in the corpus.
Within this framework, a text environment is \emph{$n$-shallow} if its corpus $r$-gram distribution coincides with the rollout law generated by its induced order-$n$ continuation law, so extending the context window beyond~$n$ yields no additional gain.
For normative publication, the fixed-point claim below is for the soft evaluator/verifier rules analysed in the appendix, not for arbitrary discontinuous hard filters.

\begin{boxedtext}
\textbf{Theorem 2 (fixed points under selection).}
Let a population of $n$-gram agents with $L$-step lookahead (equivalently, $r$-gram agents with $r = n + L$) read from and publish into a shared public corpus.

\emph{Part (a): Descriptive publication.}
When agents publish the text they generate, the set of all fixed-point $r$-gram distributions consists precisely of the \emph{$n$-shallow} distributions within the circulation polytope~$\mathcal F_r$ from Theorem~1(b): those whose corpus $r$-gram distribution coincides with the rollout law generated by their induced order-$n$ continuation law.
This set has dimension $s^{n-1}(s-1)$, independent of~$r$.
When $r = n$ it equals the full circulation polytope; when $r > n$ it is a proper, non-convex subset.
At any such fixed point the lookahead is redundant---the $n$-gram agent without lookahead would publish the same text.
Descriptive selection is self-defeating: it drives the public corpus toward $n$-shallowness, erasing the very structure that made lookahead useful.

\emph{Part (b): Normative publication.}
For the soft normative rules analysed in the appendix---those whose quality standard demands structure that no order-$n$ continuation law can produce---the fixed-point corpus is not $n$-shallow: the corpus $r$-gram distribution retains genuine structure beyond the $n$-gram window.
The Kullback--Leibler (KL) divergence between the corpus $r$-gram distribution and the rollout from its induced order-$n$ continuation law is strictly positive at the fixed point, bounded above by $L\log_2 s$ bits, where $s$ is the alphabet size.
Normative selection is self-sustaining: it maintains deep structure that rewards continued lookahead.
\end{boxedtext}

The two cases have distinct diagnostics.
In the descriptive regime, the corpus $r$-gram distribution approaches the rollout from its induced order-$n$ continuation law, so the environment becomes $n$-shallow and the KL divergence collapses to zero.
In the normative regime, first-step self-consistency can hold without full $n$-shallowness: the lookahead first-token law can approach the ordinary one-step law while the KL divergence stabilises at a strictly positive value.
The bound $L\log_2 s$ is tight: it is attained by distributing mass uniformly over cyclic windows from a de~Bruijn word of order~$n$ (see appendix, Section~5.6).
Operationally, chain-of-thought, best-of-$r$ search, and short internal rollouts can be treated here as $r$-gram publishers with $r>n$: descriptive rules recycle generated blocks, whereas normative rules score candidate futures and therefore sustain a positive KL divergence.

Greedy argmax publication can be discontinuous when the preferred continuation changes, whereas soft publication rules vary more smoothly because near-miss traces retain some weight. Figure~\ref{fig:descriptive-vs-normative} illustrates both regimes, with convergence claims for soft normative rules only under the appendix regularity conditions.

Figure~\ref{fig:descriptive-vs-normative} shows the distinction in a matched exact experiment.
In the descriptive recursion, the KL divergence between the corpus $r$-gram distribution and the rollout from its induced order-$n$ continuation law collapses essentially to zero, so the public environment becomes $n$-shallow.
In the normative recursion, the ecosystem also converges, but to a stable nonzero KL divergence.
Convergence of the ecosystem and collapse of the KL divergence between the corpus $r$-gram distribution and its continuation-law rollout are therefore different questions.

\begin{figure}[t]
\centering
\includegraphics[width=\linewidth]{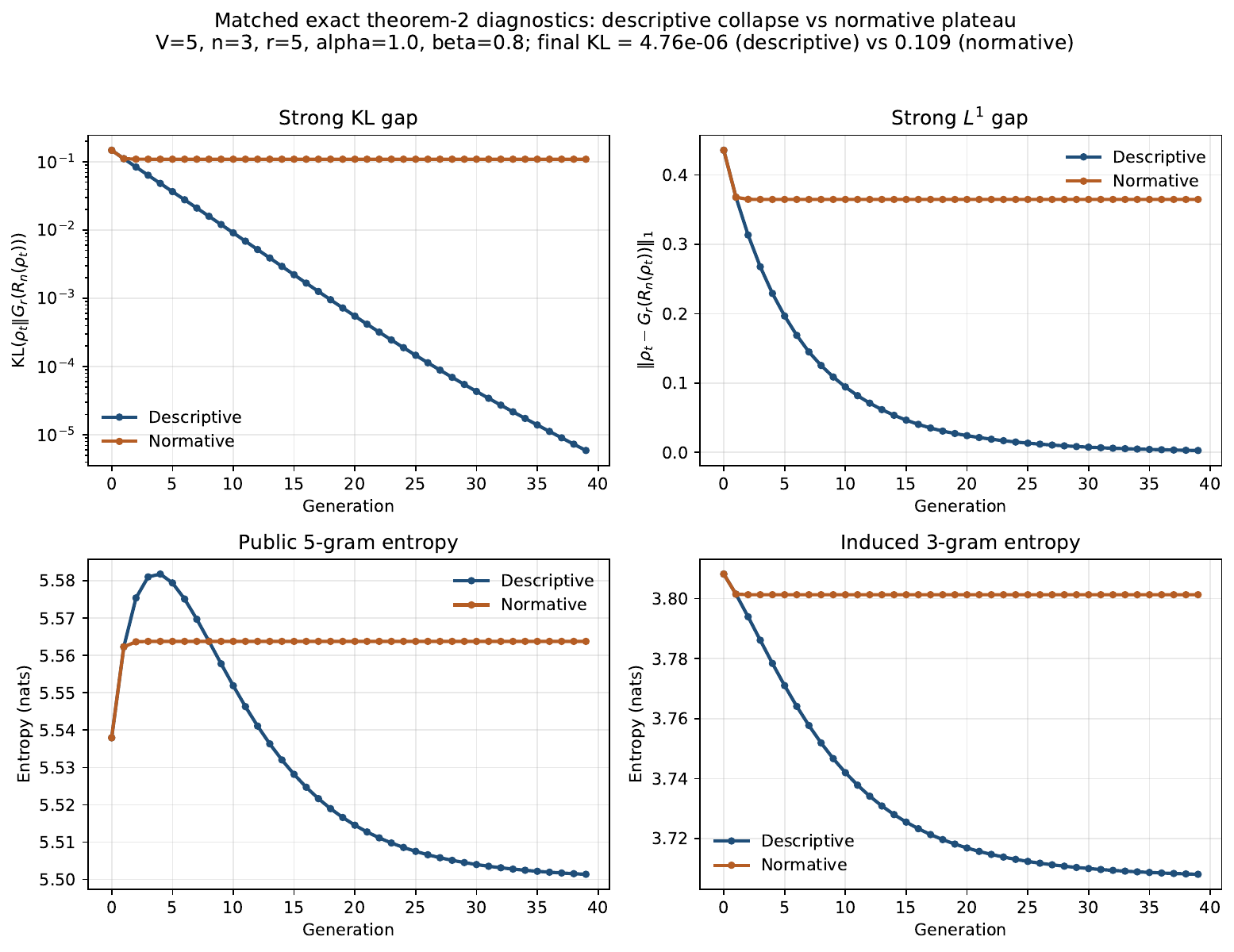}
\caption{\textbf{Descriptive versus normative publication.}
Each panel shows both recursions over $40$ generations with matched initial conditions.
In the descriptive case (blue), the Kullback--Leibler divergence (the excess of cross-entropy over entropy) between the corpus $r$-gram distribution and the rollout implied by its induced order-$n$ continuation law collapses to zero, so the public environment becomes $n$-shallow.
In the normative case (orange), the recursion converges to a stable public environment, but the KL divergence stabilises at $2.57$~bits---well within the extremal bound $L\log_2 s = 2\log_2 5 \approx 4.64$~bits---so the corpus does not become $n$-shallow.}
\label{fig:descriptive-vs-normative}
\end{figure}

Whether such compression is desirable depends on what downstream agents need.
For artefact-learning (reproducing a finished proof, a test-passing patch, or a polished explanation), filtering can remove dead ends and make successful structures easier to imitate.
For process-learning (debugging, proof search, scientific exploration), failed attempts and intermediate steps can themselves be informative.

\section*{Later learners inherit the public conditional}

\begin{boxedtext}
\textbf{Theorem 3 (cross-entropy inheritance).}
Let $q$ be the public next-token conditional generated by an environment published as a collection of texts.
Train a later learner on that environment by minimising expected next-token cross-entropy.
If the model class contains $q$, cross-entropy minimisation recovers $q$.
Otherwise, it recovers the KL-closest conditional in that class.
\end{boxedtext}

The collection-of-texts formulation matters.
For large token budgets, a single long published trajectory is not the right asymptotic object: empirical next-token frequencies along one dependent path need not recover the intended public conditional (concrete counterexamples exist).
Publishing a collection of texts is both mathematically cleaner and closer to the way real corpora are assembled.

In matched comparisons, a smoothed trigram learner and a small neural next-token learner both move toward the same target conditional. What is inherited is the public conditional; optimisation and approximation remain architecture-dependent. The same fitted $n$-gram can then publish either by sampling from that conditional or by greedy argmax publication.

If a later population publishes by sampling, it republishes $q$ itself. Under greedy argmax publication, it republishes the deterministic policy induced by $q$. The training target remains $q$, but the publication map can then be discontinuous at ties and near-ties.

Taken together, Theorems~1--3 describe how drift and selection reshape the public conditional and how later learners inherit that reshaped conditional.
Theorems~1(b) and~2 go beyond existence by giving a combinatorial characterisation of the equilibrium sets in terms of circulations on de~Bruijn graphs and the $n$-shallowness constraint.
Filtering can therefore help artefact-learning by standardising successful outputs while harming process-learning when it erases the traces needed for search.

\section*{A transparent laboratory}

The theorems above are proved in full generality, but the mechanisms they describe---support erosion under drift and probability redirection under selection---are most instructive when every intermediate quantity can be inspected.
Figure~\ref{fig:metrics-panels} serves this purpose.
It runs the recursive loop on a trigram model fitted to the public-domain fiction of Arthur Conan Doyle: a setting chosen not for scale but for transparency, since every conditional probability, every back-off event, and every selection step can be tracked exactly.

The figure is illustrative rather than confirmatory: the theorems are mathematical results, and the role of the experiment is to make their mechanisms visible in real English text.
Under neutral recursion, rare trigram contexts vanish and the generator is forced to fall back on shorter memory, producing increasingly generic continuations.
Under lookahead filtering, selection preserves short-horizon viability at the cost of diversity: it suppresses tokens that would lead into unsupported contexts, so the model can continue generating from full trigram support for longer but the text becomes more repetitive.
The trade-off between support and diversity is the signature of selection acting on a finite environment.

The appendix and repository contain broader illustrations of all three theorems, including additional drift runs, descriptive-versus-normative block-law diagnostics, and matched cross-entropy inheritance tests.
But the contribution of this paper is the theory, not any single experiment.

\begin{figure*}[t]
\centering
\includegraphics[width=0.92\textwidth]{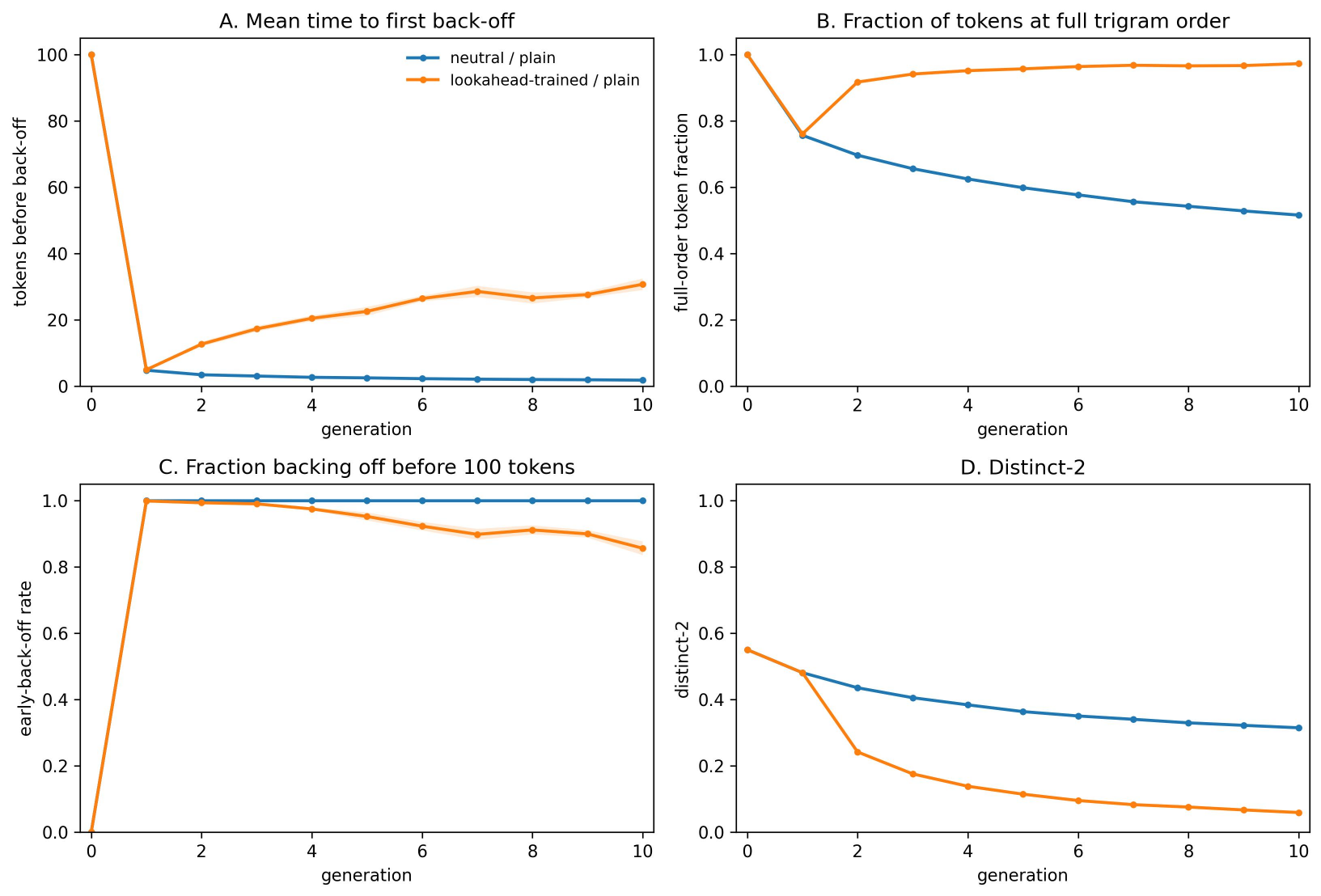}
\caption{\textbf{Recursive resampling concentrates support; filtering redirects it in the Conan Doyle corpus.}
All panels are computed from the recursive loop run on a trigram model fitted to the public-domain Arthur Conan Doyle fiction corpus used for the main-text illustration.
Under neutral recursion, finite resampling erodes weakly supported higher-order structure and pushes the generator towards more generic continuations.
Under lookahead-filtered recursion, probability mass is redirected towards continuations that preserve short-horizon viability under full-order generation.
The resulting gain in support comes at a cost in diversity, and the value of that trade-off depends on what rare contexts or traces are removed and on whether later learners need finished artefacts or the search process itself.
Parameters: $n=3$ (trigram), lookahead horizon $L=5$, generations $T=10$, total synthetic budget $M=120{,}000$ tokens per generation, five independent replicates, and $1{,}000$ fixed-seed evaluation rollouts of length $100$ per generation.
Curves show mean $\pm$ s.e.m.\ across replicates.}
\label{fig:metrics-panels}
\end{figure*}

\section*{Why this matters beyond $n$-grams}

The broader claim of this paper is about the public conditional distribution.
Recursive generation and publication filtering change the next-token law visible in the public record, and therefore change which artefacts and traces later learners can imitate.
The $n$-gram ecology makes that redistribution explicit in a setting where every quantity can be written down and analysed exactly.
Extensions with heterogeneous agents, including different temperatures and context windows, are treated in the supplementary materials.
More broadly, the circulation-polytope geometry of the unsmoothed baseline is only the first case of a richer landscape: different smoothing schemes, back-off strategies and interpolation methods each induce their own fixed-point geometry, and mapping that landscape is a natural direction for further work.

This paper does not claim that transformers literally obey variable-order $n$-gram dynamics.
The role of $n$-grams here is analogous to that of tabular Q-learning in reinforcement learning: they provide the idealised, exactly solvable case in which the theory can be stated in closed form, while neural architectures act as function-approximation proxies.
The bridge to transformers is through the public conditional itself: once an environment induces a public next-token law, later learners trained by next-token cross-entropy are pulled toward that law to the extent permitted by model class and optimisation.\cite{Bengio2003NNLM}
Real LLM ecosystems add forces absent here---richer curation rules, reward-based filtering,\cite{Ouyang2022InstructGPT} and platform amplification\cite{Huszar2022Amplification,Dujeancourt2023}---but the basic mechanism of drift, selection and inheritance persists beyond the $n$-gram setting.

\section*{Methods summary}

The illustrations in this paper use unsmoothed variable-order $n$-gram models fitted by maximum likelihood, backing off only when the full-order count is zero.\cite{Katz1987,KneserNey1995}
Each generation publishes a collection of synthetic texts rather than a single long trajectory, refits the model, and repeats.
Parameter settings for the Conan Doyle figures are given in the caption of Figure~\ref{fig:metrics-panels}.
Full proofs, exact diagnostics for Theorem~2, additional source-text and $n$-gram experiments, arithmetic-environment comparisons, and cross-architecture inheritance tests are provided in the appendix and repository at \url{\repoURL}.

\section*{Data availability}
The Conan Doyle source texts used for the main-text figures are public-domain Project Gutenberg editions.
Additional appendix analyses use public-domain texts by Jane Austen and Charles Darwin.
Cleaned corpus snapshots, train/validation/test splits, generated artefacts, summary CSV files, and repository manifests are available at \url{\repoURL}.

\section*{Code availability}
The exact code used to generate all figures and tables in this paper, together with the notebook pipelines and supporting modules, is available at \url{\repoURL}.
The main notebooks and helper modules are located in the repository's \texttt{notebooks/} and \texttt{src/} directories.

\section*{Acknowledgements}
The author thanks colleagues for feedback after an early preliminary version of this work was presented at an Informal Meeting on 6 March 2026, at the Centre for Fundamentals of AI and Computational Theory, Queen Mary University of London.

\section*{Author contributions}
S.R. conceived the study, discovered the mathematical relationships, developed the theorems, and wrote the manuscript.

\section*{Competing interests}
The author declares no competing interests.

\balance

\clearpage
\newgeometry{margin=1in}
\setlength{\textfloatsep}{12pt plus 2pt minus 2pt}
\setlength{\floatsep}{12pt plus 2pt minus 2pt}
\setlength{\intextsep}{12pt plus 2pt minus 2pt}
\setlength{\abovecaptionskip}{6pt}

\begin{center}
{\Large\bfseries Appendix: A Guided Route Through the Results}\\[0.5em]
{\large Drift and selection in LLM text ecosystems}\\[0.8em]
{S\o ren Riis}\\
{\small Queen Mary University of London, London, United Kingdom}\\
{\small \texttt{s.riis@qmul.ac.uk}}
\end{center}
\vspace{1em}

\noindent\textbf{Abstract.}\enspace
This appendix accompanies the main paper \emph{Drift and selection in LLM text ecosystems}.
It is organised as seven cumulative sections that develop the mathematical core and illustrate it with select experiments.
Section~1 introduces the $n$-gram model, the urtext recursion, and shows experimentally how vocabulary contracts under finite sampling; it states Theorem~1 part~(a) (drift) in full, including the Wright--Fisher identification and effective population size.
Section~2 takes the complementary limit $M \to \infty$, where the distributional recursion on $n$-gram laws becomes exact, and gives a complete characterisation of all fixed points as circulations on de~Bruijn graphs (Theorem~1 part~(b)), with fully worked examples from bigrams through $4$-grams and the general circulation-polytope proof.
Section~3 develops the machinery for projecting fixed points to different window lengths and introduces the distinction between shallow and deep distributions.
Section~4 introduces agent populations---ordinary $n$-gram agents and lookahead $r$-gram agents---and develops the selection theory (Theorems~2 and~3), including the distinction between descriptive and normative publication rules.
Section~5 isolates the information-theoretic diagnostics for Theorem~2 and shows, in matched exact runs, how descriptive recursion drives the KL divergence and $L^1$ distance to zero while a normative publication rule can converge to a stable nonzero KL divergence.
Section~6 presents basic results for the more general setting in which text is produced by a large population of agents with different context windows and publication temperatures, extending the single-agent framework of Sections~1--5.
Section~7 provides an overview table of the experimental matrix.
Each section includes pointers to the repository notebooks and scripts where the reader can reproduce figures and tables.
The full repository is at \url{\repoURL}.

\tableofcontents

\newpage
\section{The $n$-gram model, drift, and vocabulary loss under recursion}\label{sec:ngram-basics}

\subsection{Orientation}

This section introduces the basic setup: a finite text corpus, a variable-order $n$-gram model fitted to it, and the recursive loop in which the model generates new text that becomes the corpus for the next generation.
The goal in this section is to see, in the simplest possible setting, how recursive resampling erodes rare structure, and to develop the distributional theory that yields Theorem~1 (drift).
This section covers the following:
\begin{itemize}[leftmargin=*]
\item the definition of a variable-order $n$-gram model with back-off (falling back to shorter contexts when the full context is unseen);
\item the fit--generate--refit recursion from a fixed starting corpus (the \emph{urtext}---the original human-written text from which the recursion begins);
\item vocabulary contraction across generations, measured experimentally;
\item the mechanism by which rare words and rare phrases are the first casualties; and
\item the mixed-environment unigram recursion, Theorem~1 (drift) in full, and why the replacement fraction $\alpha$ controls the speed but not the destination.
\end{itemize}

\subsection{The token set and corpus}

\begin{definition}[Token set and text]
Let $\Sigma$ be a finite set of tokens, called the \emph{token set} (or \emph{vocabulary}).
In a word-level model, each element of $\Sigma$ is a word; in a character-level model, each element is a character.
A \emph{text} (or \emph{corpus}) of length $M$ is a sequence
\[
T = x_1 x_2 \cdots x_M \in \Sigma^M.
\]
\end{definition}

\begin{example}[Conan Doyle]
Take $\Sigma$ to be the set of distinct words appearing in the combined public-domain fiction of Arthur Conan Doyle (Project Gutenberg editions).
Tokenisation is word-level: punctuation, page numbers and Gutenberg boilerplate are stripped; hyphenated words (e.g.\ ``well-known'') and contractions (e.g.\ ``don't'') are kept as single tokens; all tokens are lowercased.
The corpus $T$ is the concatenation of those texts, with $M \approx 4{,}013{,}000$ tokens and $|\Sigma| \approx 55{,}000$ distinct words.
We use Doyle throughout this section as the primary running example, with Jane Austen ($M \approx 601{,}000$ tokens, $|\Sigma| \approx 14{,}000$) and Charles Darwin ($M \approx 356{,}000$ tokens, $|\Sigma| \approx 15{,}600$) as supplementary corpora.
\end{example}

\subsection{The variable-order $n$-gram model}

The idea is simple: to predict the next word, look at the preceding $n-1$ words and use the empirical frequencies from the corpus.

\begin{definition}[Context and $n$-gram counts]\label{def:ngram}
Fix a maximum order $n \ge 1$.
For every context string $s \in \Sigma^m$ with $0 \le m \le n-1$ and every token $a \in \Sigma$, let
\[
N_T(s, a)
\]
be the number of times $a$ follows $s$ in the corpus $T$.
The \emph{empirical conditional probability} is
\[
\widehat{p}_T(a \mid s) = \frac{N_T(s,a)}{\sum_{b \in \Sigma} N_T(s,b)},
\]
whenever the denominator is nonzero.
\end{definition}

\begin{definition}[Back-off]
If the full-order context (the preceding $n-1$ tokens) has never been seen in the corpus, the model \emph{backs off} to a shorter context: it drops the oldest token and tries the $(n-2)$-gram context, then the $(n-3)$-gram context, and so on, down to the unigram (empty context) distribution.
This is the variable-order $n$-gram generator used throughout.
\end{definition}

\begin{example}[A trigram model, $n=3$]
With $n=3$, the model looks at the two preceding words (a \emph{bigram context}).
To generate the next word after ``\texttt{Sherlock Holmes}'', it looks up all words that ever followed ``\texttt{Sherlock Holmes}'' in the corpus and samples in proportion to their counts.
If ``\texttt{Sherlock Holmes}'' never appeared, it backs off to the unigram context of ``\texttt{Holmes}'' alone; if even that is absent, it falls back to the global word-frequency distribution.
We use trigrams ($n=3$) in most experiments because they are the shortest context length that already exhibits interesting phrase-level structure while remaining simple enough for exact computation.
\end{example}

\paragraph{Why $n$-grams?}
We do not claim that $n$-gram models are realistic surrogates for modern transformers.
The point is that they are the simplest next-token predictors where the entire conditional distribution is exactly computable and easy to inspect.
The relationship is analogous to that between tabular Q-learning and deep reinforcement learning: tabular Q-values are the idealised, exactly solvable objects, and neural networks serve as function-approximation proxies whose convergence guarantees ultimately rest on the tabular theory.
Here, $n$-grams play the tabular role---they are next-token predictors for which the ecosystem dynamics can be written in closed form---and the theorems proved for them identify forces (drift, selection, inheritance) that are structural properties of any next-token prediction loop, not artefacts of a particular architecture.
If these forces already produce non-trivial dynamics in $n$-grams, then they live at the level of distributions and of which traces enter the environment, not at the level of transformer internals.

\subsection{The urtext and the recursive loop}\label{sec:tut1-urtext}

\begin{definition}[Urtext]
The starting corpus $U_0 \in \Sigma^M$ is called the \emph{urtext}.
It can be a fixed human-written text or a stochastic randomly generated corpus; it is the initial sequence from which the recursion begins.
Once $U_0$ and the model order $n$ are fixed, the law of the stochastic recursion is fully determined by one additional parameter: the replacement fraction $\alpha \in [0,1]$.
\end{definition}

\begin{definition}[The fit--generate--refit loop]\label{def:loop}
At each generation $t = 0, 1, 2, \ldots$:
\begin{enumerate}[label=(\roman*)]
\item \textbf{Fit.} Compute the empirical $n$-gram kernel $\widehat{p}_t$ from the current corpus $U_t$.
\item \textbf{Keep.} Retain a fraction $(1-\alpha)$ of $U_t$ unchanged.
\item \textbf{Generate.} Sample a fresh block of $\alpha M$ tokens from the fitted model $\widehat{p}_t$.
\item \textbf{Concatenate.} The kept block and the generated block together form $U_{t+1}$.
\end{enumerate}
\end{definition}

The parameter $\alpha$ controls the speed of change.
When $\alpha = 1$, the entire corpus is replaced each generation (the most aggressive setting).
When $\alpha = 0.25$, three-quarters of the corpus is carried over unchanged, and only one quarter is regenerated.
In both cases, the regenerated text is sampled from the model fitted to the \emph{current} corpus, so the loop is self-referential: models learn from text that earlier models helped produce.

\begin{remark}[The population-genetics analogy]
The recursion has a direct analogy in population genetics.
Words play the role of alleles, the corpus is the gene pool, and retraining on a finite regenerated sample creates the next generation.
With $\alpha = 1$ and at the unigram level, this is \emph{exactly} the Wright--Fisher model from population genetics---the foundation of neutral drift theory.
We develop this connection rigorously below.
\end{remark}

\subsection{Why rare words disappear first}

Consider a word that appears with frequency $p$ in the current corpus.
If the regenerated block contains $M$ tokens drawn independently, each with probability $p$ of being that word, then the probability that the word is \emph{completely absent} from the regenerated block is
\begin{equation}\label{eq:appx-dropout}
\Pr(\text{word absent from regenerated block}) = (1 - p)^M.
\end{equation}
For a common word with $p = 0.01$ and $M = 100{,}000$, this is essentially zero.
For a rare word with $p = 1/M$, this is $(1 - 1/M)^M \approx 1/e \approx 0.37$.
So a word that appears just once in a corpus of $100{,}000$ tokens has roughly a $37\%$ chance of vanishing in a single generation---and in an unsmoothed model (one that assigns \emph{zero} probability to words it has never seen), once a word vanishes it can never return.

This is the basic mechanism of drift: finite sampling from a distribution in which rare events have small probability.

\subsection{Experiment: vocabulary contraction in Doyle, Austen, and Darwin}

The repository contains a notebook that runs the recursive loop on three cleaned literary corpora: Arthur Conan Doyle, Jane Austen, and Charles Darwin.
Each corpus is modelled with an unsmoothed word-level trigram model ($n=3$) and the mixed-replacement update of Definition~\ref{def:loop} (\texttt{fixed\_size\_mix} in the code).
Three replacement fractions are tested: $\alpha \in \{0.25, 0.5, 1.0\}$, over $12$ generations.

\paragraph{What to look for.}
Two primary metrics are tracked at each generation:
\begin{itemize}[leftmargin=*]
\item \textbf{Vocabulary retention}: the fraction of the original word types (generation~0) that still have positive count at generation~$t$.
\item \textbf{Trigram-type retention}: the fraction of distinct trigram types from generation~0 that survive at generation~$t$.
\end{itemize}
An auxiliary metric tracks the \emph{active} vocabulary---word types whose frequency exceeds the baseline threshold $1/M_0$.

\paragraph{What the figures show.}
Figure~\ref{fig:appx-theorem1-vocab-retention} shows that vocabulary retention declines monotonically, with stronger replacement fractions producing steeper declines.
At full replacement ($\alpha = 1$) after $12$ generations, roughly half the vocabulary is lost across all three corpora: retention is $0.473$ for Doyle, $0.449$ for Austen, and $0.408$ for Darwin.

\begin{figure}[t]
\centering
\includegraphics[width=0.98\linewidth]{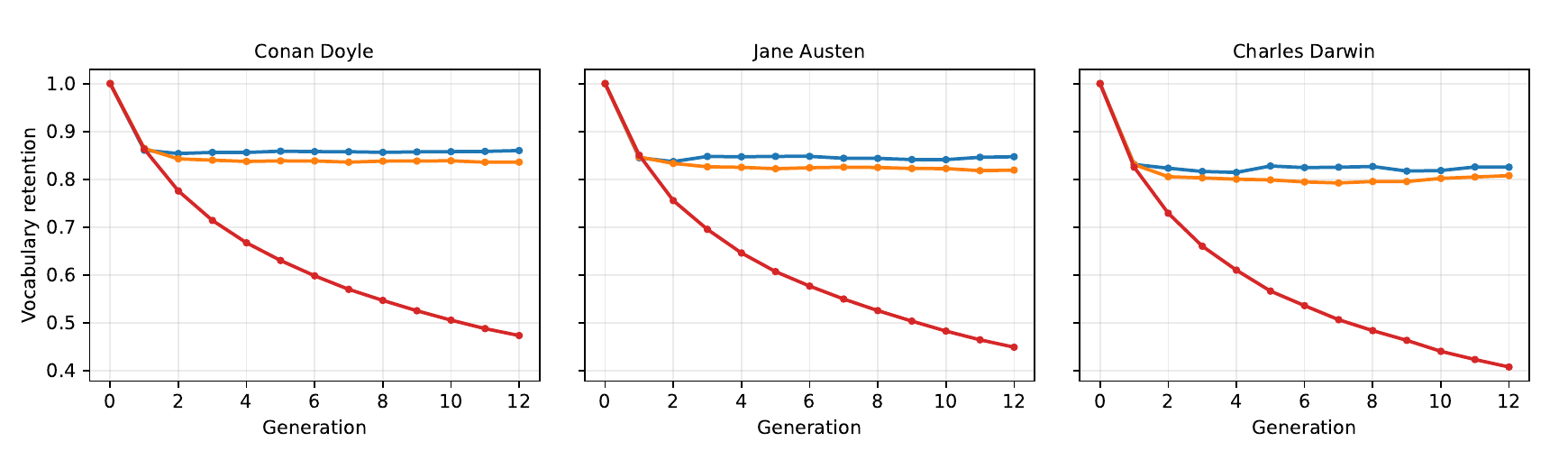}
\caption{\textbf{Vocabulary retention declines under recursive resampling.}
Across Doyle, Austen, and Darwin, stronger replacement fractions produce stronger contraction.
Full replacement ($\alpha=1$) yields the steepest decline by generation~12.}
\label{fig:appx-theorem1-vocab-retention}
\end{figure}

Trigram-type retention falls even faster (Figure~\ref{fig:appx-theorem1-trigram-retention}): down to $0.211$, $0.169$, and $0.165$ for Doyle, Austen, and Darwin respectively at $\alpha = 1$.
This makes sense: a trigram can only survive if all three of its component words survive \emph{and} the specific three-word sequence is regenerated.
Higher-order structure is more fragile than lower-order structure, because there are more ``points of failure.''

\begin{figure}[t]
\centering
\includegraphics[width=0.98\linewidth]{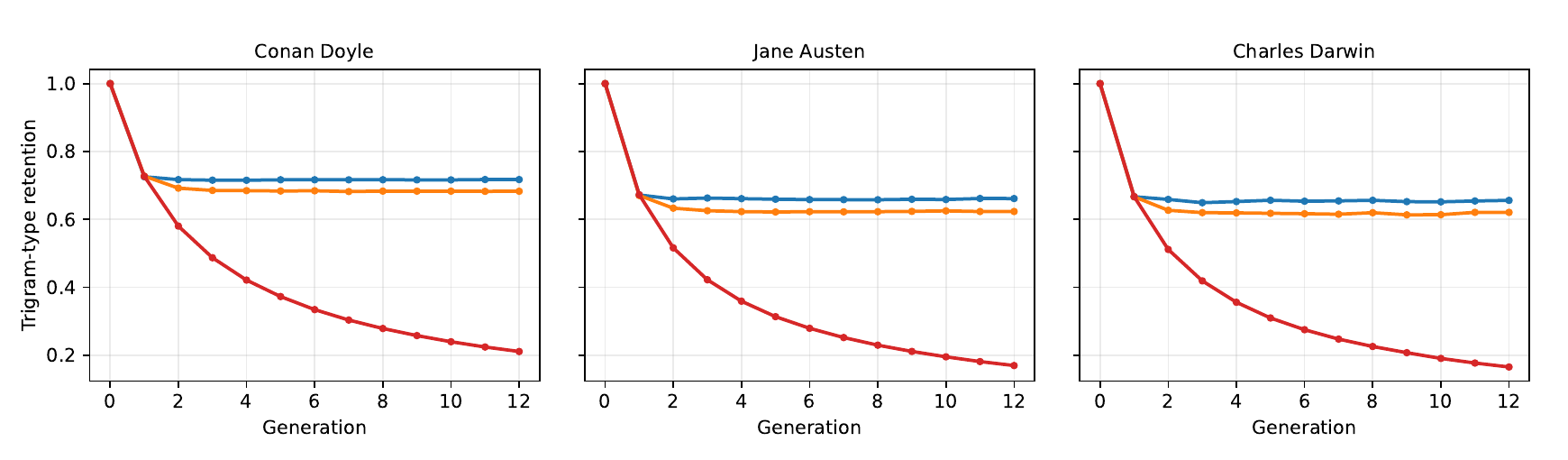}
\caption{\textbf{High-order support contracts faster than vocabulary.}
Trigram-type retention falls more sharply than vocabulary retention, showing the fragility of higher-order continuation structure.}
\label{fig:appx-theorem1-trigram-retention}
\end{figure}

\begin{figure}[t]
\centering
\includegraphics[width=0.98\linewidth]{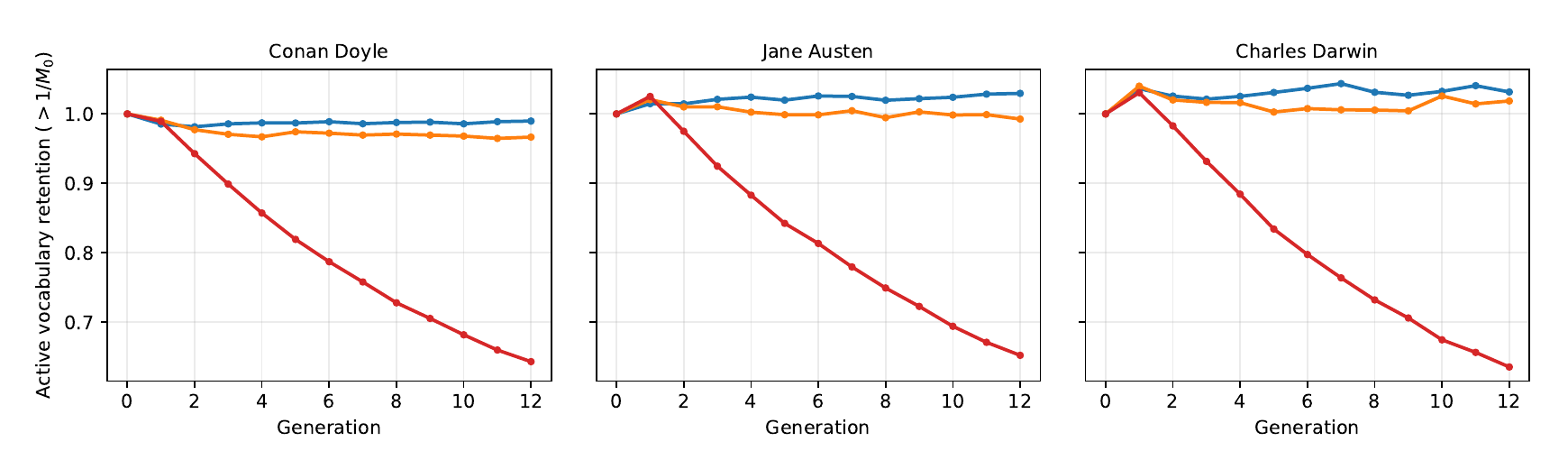}
\caption{\textbf{Active-vocabulary retention above the baseline threshold $1/M_0$.}
Even thresholded active vocabulary shows the same monotone concentration pattern as $\alpha$ increases.}
\label{fig:appx-theorem1-active-vocab}
\end{figure}

\begin{table}[t]
\centering
\caption{Theorem~1 negative endpoints under fixed-size recursive trigram resampling (generation 12).}
\label{tab:appx-theorem1-negative-endpoints}
\begin{tabular}{llccc}
\toprule
Corpus & $\alpha$ & Generation & Vocabulary retention & Trigram-type retention \\
\midrule
Charles Darwin & 0.25 & 12 & 0.825 & 0.656 \\
Charles Darwin & 0.50 & 12 & 0.808 & 0.621 \\
Charles Darwin & 1.00 & 12 & 0.408 & 0.165 \\
Conan Doyle & 0.25 & 12 & 0.860 & 0.717 \\
Conan Doyle & 0.50 & 12 & 0.836 & 0.683 \\
Conan Doyle & 1.00 & 12 & 0.473 & 0.211 \\
Jane Austen & 0.25 & 12 & 0.847 & 0.661 \\
Jane Austen & 0.50 & 12 & 0.819 & 0.623 \\
Jane Austen & 1.00 & 12 & 0.449 & 0.169 \\
\bottomrule
\end{tabular}
\end{table}

\paragraph{The positive side of drift.}
The same concentration pressure that deletes rare literary vocabulary can also suppress low-support orthographic variation.
In an Austen orthographic-variant pilot, low-support spelling variants are injected into a clean corpus and then subjected to the same recursive resampling.
The mechanism that removes rare forms also reduces the variant count while increasing the canonical form's share---but it simultaneously erodes rare clean vocabulary.
This is why Theorem~1 is best read as a theorem of \emph{concentration} rather than a theorem of simple deterioration.

\begin{figure}[t]
\centering
\includegraphics[width=0.98\linewidth]{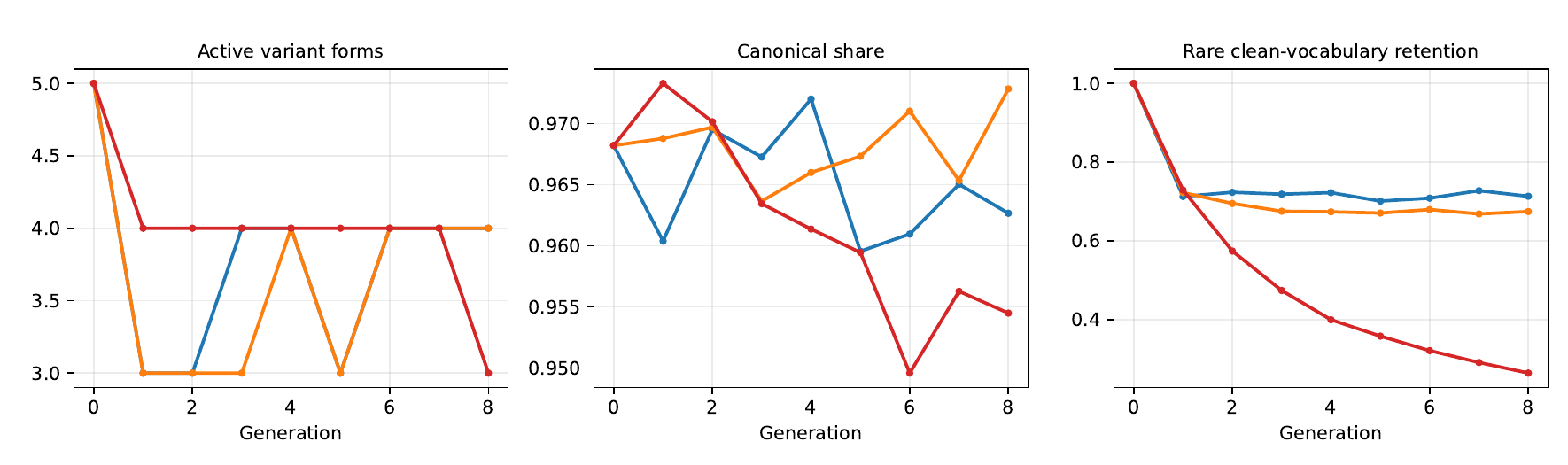}
\caption{\textbf{The positive side: orthographic standardisation in an Austen pilot (0.5\% injection rate).}
Active variant forms decline after generation~1; canonical share rises; but rare clean-vocabulary retention also falls.}
\label{fig:appx-theorem1-orthographic-panels}
\end{figure}

\subsection{Rare contexts and the fragility of higher-order support}

The vocabulary-retention plots show the overall effect, but it is worth understanding the mechanism more precisely.
The key insight is that what disappears first is often not the word itself but the \emph{highest-order support} that allows a specific phrase to be regenerated.

\begin{proposition}[Rare-pattern loss under finite sampling]\label{prop:tut1-extinction}
Let $g$ be a $k$-gram with $1 \le k \le n$ and long-run frequency $\pi(g) > 0$ under the current model.
Generate a synthetic corpus consisting of one or more texts with a combined length of $M$ tokens, and fit an unsmoothed maximum-likelihood $n$-gram model.
Then:
\begin{enumerate}[label=(\alph*)]
\item If $k=1$ and the corpus consists of $M$ i.i.d.\ unigram draws, the exact dropout probability is
\[
\Pr(g\text{ absent}) = (1 - \pi(g))^M.
\]
\item For general $k \le n$, if $N_g$ is the count of $g$ in the synthetic corpus, then
\[
\E[N_g] \le (M - k + 1)\pi(g),
\qquad
\Pr(N_g = 0) \ge \max\{0,\; 1 - (M-k+1)\pi(g)\}.
\]
The expectation is exact when the corpus is a single sequence; when it consists of $J$ separate texts, $k$-grams cannot span text boundaries, so the effective number of positions is at most $M - J(k-1)$.
\item If $g = (s, a)$ is a full-order pattern with $|s| = n-1$ and $N_g = 0$, then the retrained model satisfies $\widehat{p}(a \mid s) = 0$.
In an unsmoothed model without back-off at that context, this loss is absorbing: the token $a$ can never again follow context $s$.
\end{enumerate}
\end{proposition}

\begin{proof}
Part~(a) is immediate from independence.
Part~(b) follows from the first-moment bound (Markov's inequality): $\E[N_g] \ge \Pr(N_g \ge 1)$.
Part~(c) is the definition of unsmoothed maximum likelihood.
\end{proof}

Back-off complicates permanent loss: a missing higher-order pattern can still be generated through a shorter context.
But back-off turns \emph{distinctive} prose into more \emph{generic} prose, because the shorter context is shared with many other continuations.
Common forms therefore enjoy a compounding advantage over rare forms across generations.

\subsection{The $n$-gram model as a Markov chain on contexts}

Before stating the drift theorem, it is useful to see that any fixed-order $n$-gram model defines a Markov chain on contexts.
This viewpoint will be essential in the next section for the selection theory.

\begin{proposition}[Context Markov chain]\label{prop:tut2-markov}
Fix $n \ge 2$ and a conditional distribution $p(a \mid c)$ giving the probability of emitting token $a$ after context $c$, where $\mathcal{C} = \Sigma^{n-1}$ is the set of all contexts of length $n-1$.
Let $C_t$ be the current context and $A_{t+1}$ the next emitted token.
Write $\sigma(c, a)$ for the \emph{shift map}: the context obtained by appending token $a$ to context $c$ and dropping the oldest token.
The successor context is $C_{t+1} = \sigma(C_t, A_{t+1})$.
Then $\{C_t\}_{t \ge 0}$ is a time-homogeneous Markov chain on $\mathcal{C}$ with transition matrix
\[
K_p(c, c') = \sum_{a:\, \sigma(c,a) = c'} p(a \mid c).
\]
If $K_p$ is irreducible and aperiodic (meaning every context can eventually be reached from any other, and the chain does not cycle with a fixed period), it has a unique stationary distribution $\pi_p$, and the empirical context frequencies along a single trajectory converge almost surely to $\pi_p$.
\end{proposition}

\begin{proof}
Given $C_t = c$, the next token is drawn from $p(\cdot \mid c)$ and the next context is determined by the shift map $\sigma$.
This is exactly the Markov property.
Convergence follows from the ergodic theorem for finite Markov chains~\cite{appx:Norris1997}.
\end{proof}

\paragraph{What this means.}
Once an $n$-gram kernel has been fitted, the text dynamics are those of a finite Markov chain on contexts.
The urtext matters only through the fitted kernel and the initial context distribution.
This is the key simplification that makes the selection theory in Section~\ref{sec:selection-theory} tractable.

\subsection{The mixed-environment unigram recursion and Theorem 1}

We now set up the formal model for Theorem~1 on drift.
The key object is the number of tokens belonging to a designated ``minority'' vocabulary class.

\begin{definition}[Minority count chain]\label{def:count-chain}
Fix a corpus of $M$ token positions and a designated minority subset $R \subset \Sigma$ (a single rare word, a dialect cluster, or any vocabulary subset of interest).
Let $n(t) \in \{0, 1, \ldots, M\}$ be the number of positions occupied by tokens from $R$ at generation $t$, and write $\mu_t = n(t)/M$ for the minority frequency.
\end{definition}

\begin{definition}[One-generation transition]\label{def:one-gen}
At each generation, the corpus is updated in two steps:
\begin{enumerate}[label=\textbf{Step~\arabic*},leftmargin=*]
\item \textbf{Retention.}
Of the $M$ positions, exactly $(1-\alpha)M$ are chosen uniformly at random (without replacement) to be \emph{kept} unchanged.
The remaining $\alpha M$ positions are marked for replacement.

\item \textbf{Resampling.}
Each of the $\alpha M$ replacement positions is filled independently by drawing a token from the \emph{current} empirical distribution: minority with probability $\mu_t = n(t)/M$, majority with probability $1 - \mu_t$.
Once a minority token is lost from the entire corpus, it cannot reappear.
\end{enumerate}
\end{definition}

Conditional on $n(t) = n$, the new count decomposes as
\begin{equation}\label{eq:decomp}
n(t+1) = J + K,
\end{equation}
where $J$ and $K$ are conditionally independent given $n(t)$:
\begin{itemize}[leftmargin=*]
\item $J$ is the number of minority tokens among the $(1-\alpha)M$ kept positions. Since these are drawn without replacement from a corpus of $M$ tokens of which $n$ are minority, $J$ follows a hypergeometric distribution:
\[
\Pr(J = j) = \frac{\binom{n}{j}\binom{M-n}{(1-\alpha)M - j}}{\binom{M}{(1-\alpha)M}}, \qquad \E[J] = (1-\alpha)\,n.
\]
\item $K$ is the number of minority tokens among the $\alpha M$ newly drawn positions. Each is drawn independently with probability $n/M$ of being minority, so $K$ follows a binomial distribution:
\[
\Pr(K = k) = \binom{\alpha M}{k}\Bigl(\frac{n}{M}\Bigr)^k \Bigl(1 - \frac{n}{M}\Bigr)^{\alpha M - k}, \qquad \E[K] = \alpha\, n.
\]
\end{itemize}
The boundary states $n = 0$ and $n = M$ are absorbing: once the minority has been completely lost (or has taken over the entire corpus), it stays that way.

\subsection{Theorem 1 (drift in the mixed environment)}

\begin{theorem}[Drift in the mixed environment]\label{thm:drift}
Consider the count chain $\{n(t)\}_{t \ge 0}$ on $\{0, 1, \ldots, M\}$ defined above, and write $\mu_t = n(t)/M$.
Then:

\begin{enumerate}[label=(\alph*)]
\item\label{item:martingale} \emph{(Unbiased expected value.)}
$\E[n(t+1) \mid n(t) = n] = n$.
Equivalently, $\E[\mu_{t+1} \mid \mu_t] = \mu_t$.
The expected minority frequency next generation equals its current value---there is no systematic upward or downward trend.
(In probability theory, a process with this property is called a \emph{martingale}.)
Despite the absence of any trend, random fluctuations accumulate over time and eventually drive the count to one of the absorbing boundaries.

\item\label{item:variance} \emph{(Variance.)}
In the full discrete model:
\[
\mathrm{Var}(n(t+1) \mid n(t) = n)
=
\underbrace{(1-\alpha)\alpha\,\frac{Mn(M-n)}{M(M-1)}}_{\text{hypergeometric (retention)}}
+
\underbrace{\alpha\,n\bigl(1 - n/M\bigr)}_{\text{binomial (resampling)}}.
\]
For large $M$ with $\mu = n/M$ fixed, this simplifies to $\approx \alpha(2-\alpha)\,n(1-\mu)$.

\item\label{item:extinction} \emph{(Extinction probability.)}
Let $\tau = \inf\{t \ge 0 : n(t) \in \{0, M\}\}$ be the first generation at which the minority has either completely vanished or taken over the entire corpus.
Then $\tau < \infty$ almost surely, and
\[
\Pr\bigl(n(\tau) = 0 \mid n(0) = k\bigr) = 1 - \frac{k}{M}.
\]
This is \textbf{independent of $\alpha$}: the replacement fraction does not affect whether a rare token eventually goes extinct, only how quickly.

\item\label{item:Neff} \emph{(Effective population size.)}
In the diffusion scaling (where time is measured in units of $N_e$ generations), the effective population size is
\[
N_e = \frac{M}{2\alpha(2-\alpha)}.
\]
When $\alpha = 1$, this gives $N_e = M/2$, the classical Wright--Fisher value (the Wright--Fisher model being the foundational model of neutral genetic drift in population genetics~\cite{appx:Ewens2004}).

\item\label{item:WF} \emph{(Wright--Fisher identification.)}
When $\alpha = 1$, every position is replaced and $J = 0$ a.s., so $n(t+1) = K \sim \mathrm{Bin}(M, n(t)/M)$.
This is exactly the classical Wright--Fisher model.

\item\label{item:dropout} \emph{(Single-step dropout.)}
For a single minority token ($n = 1$):
\[
\Pr\bigl(n(t+1) = 0 \mid n(t) = 1\bigr) = \alpha \cdot \bigl(1 - 1/M\bigr)^{\alpha M} \approx \alpha\, e^{-\alpha}.
\]
\end{enumerate}
\end{theorem}

\begin{proof}
\ref{item:martingale}~\emph{(Unbiased expected value.)}
The expected number of minority tokens retained is $\E[J \mid n] = (1-\alpha)n$ (the mean of the hypergeometric), and the expected number redrawn as minority is $\E[K \mid n] = \alpha M \cdot (n/M) = \alpha n$ (the mean of the binomial).
Adding: $\E[n(t+1) \mid n(t) = n] = (1-\alpha)n + \alpha n = n$.

\ref{item:variance}~\emph{(Variance.)}
Since $J$ and $K$ are conditionally independent,
$\mathrm{Var}(n(t+1) \mid n) = \mathrm{Var}(J \mid n) + \mathrm{Var}(K \mid n)$.
The hypergeometric variance is $\mathrm{Var}(J \mid n) = (1-\alpha)M \cdot \frac{n}{M} \cdot \frac{M-n}{M} \cdot \frac{M - (1-\alpha)M}{M - 1} = (1-\alpha)\alpha\,\frac{n(M-n)}{M(M-1)}$.
The binomial variance is $\mathrm{Var}(K \mid n) = \alpha M \cdot \frac{n}{M} \cdot (1 - \frac{n}{M}) = \alpha\,n(1-n/M)$.
For large $M$ the hypergeometric term simplifies to $(1-\alpha)\alpha\,n(1-\mu)$, so the total is $\approx \alpha(2-\alpha)\,n(1-\mu)$.
The variance does not directly determine the expected waiting time for extinction, but it controls the \emph{rate} at which the random walk spreads: a larger per-step variance means faster movement toward the absorbing boundaries.

\ref{item:extinction}~\emph{(Extinction.)}
By~\ref{item:martingale}, the expected value of $n(t)$ never changes.
The variance in~\ref{item:variance} is strictly positive whenever $0 < n < M$, so the count keeps fluctuating and cannot remain in the interior forever: it must eventually reach $\{0, M\}$, i.e.\ $\tau < \infty$ almost surely.
Because the expected value is preserved at every step, it is also preserved at the stopping time $\tau$ (this is a standard result in probability theory, known as the optional stopping theorem for bounded processes).
So $\E[n(\tau)] = n(0) = k$.
At time $\tau$ the count is either $0$ or $M$, so writing $p = \Pr(n(\tau) = 0 \mid n(0) = k)$ we get $\E[n(\tau)] = 0 \cdot p + M \cdot (1 - p) = M(1-p) = k$, which gives $p = 1 - k/M$.

\ref{item:Neff}~\emph{(Effective population size.)}
By definition, $N_e$ satisfies $\mathrm{Var}(\mu_{t+1} \mid \mu_t) = \mu_t(1-\mu_t)/(2N_e)$.
Dividing the variance of $n(t+1)$ by $M^2$ and comparing gives $N_e = M/[2\alpha(2-\alpha)]$.

\ref{item:WF}~\emph{(Wright--Fisher.)}
When $\alpha = 1$, every position is replaced, so $J = 0$ and $n(t+1) = K \sim \mathrm{Bin}(M, n/M)$, which is the definition of the Wright--Fisher model.

\ref{item:dropout}~\emph{(Single-step dropout.)}
For $n = 1$, the token is lost when it is neither retained nor redrawn.
It fails to be retained with probability $\alpha$ (the chance its single position is among the replaced ones).
Conditional on not being retained, it must also not be redrawn in any of the $\alpha M$ new draws, each of which selects it with probability $1/M$: this happens with probability $(1 - 1/M)^{\alpha M} \approx e^{-\alpha}$.
Combining: $\Pr(n(t+1) = 0 \mid n(t) = 1) = \alpha(1-1/M)^{\alpha M} \approx \alpha e^{-\alpha}$.
\end{proof}

\paragraph{What $\alpha$ controls and what it does not.}
This clean separation is one of the most useful consequences of the theorem:

\begin{center}
\begin{tabular}{@{}lcc@{}}
\toprule
\textbf{Quantity} & \textbf{Depends on $\alpha$?} & \textbf{Formula} \\
\midrule
Extinction probability & No & $1 - k/M$ \\
Fixation probability & No & $k/M$ \\
Martingale property & No & $\E[n' \mid n] = n$ \\
\midrule
One-step dropout ($n=1$) & Yes & $\approx \alpha\, e^{-\alpha}$ \\
Per-step variance & Yes & $\alpha(2-\alpha)\,n(1-\mu)$ \\
Effective population size & Yes & $M/[2\alpha(2-\alpha)]$ \\
\bottomrule
\end{tabular}
\end{center}

In words: every rare token will eventually be lost with probability $1 - k/M$, \emph{no matter what $\alpha$ is}.
Conversely, a selectively neutral token that currently occupies $k$ out of $M$ positions will, through drift alone, eventually take over the entire corpus with probability $k/M$.
This may seem counter-intuitive: pure chance, with no selective advantage, can drive a token to complete dominance.
For a token appearing just once ($k = 1$), the fixation probability is $1/M$---tiny but nonzero---while the extinction probability is $1 - 1/M$, overwhelmingly close to~$1$.
This is the same mathematics that governs neutral alleles in population genetics.

Smaller $\alpha$ slows the random walk (larger $N_e$), but it does not prevent extinction---it only delays it.
Indeed, as Section~\ref{sec:selection-theory} will show, selection pressures can \emph{accelerate} the loss of certain patterns, making extinction even more inevitable than the neutral theory predicts.

\subsection{Section 1 in brief}

An $n$-gram model fitted to a fixed corpus, together with a replacement fraction $\alpha$, defines a deterministic recursive loop.
Finite resampling causes rare words and rare phrases to disappear, with higher-order structure ($k$-grams for larger $k$) being the most fragile.
The same concentration pressure can be beneficial (standardising spelling variants) or harmful (deleting rare literary vocabulary), depending on what is rare and whether it is valuable.
The minority-frequency process is a martingale: its expected value is unchanged, but variance accumulates and drives rare forms to extinction.
The extinction probability $1 - k/M$ is universal---independent of $\alpha$---and only the speed depends on $\alpha$.
The Wright--Fisher identification ($\alpha = 1$) connects the recursion to one of the best-understood stochastic processes in mathematical biology.
The complementary question---what happens in the infinite-corpus limit $M \to \infty$, where sampling noise vanishes and the distributional recursion $\rho_{t+1} = G_n(R_n(\rho_t))$ becomes exact---is treated in Section~\ref{sec:worked-examples}, which gives a complete characterisation of all fixed points as circulations on de~Bruijn graphs (Theorem~\ref{thm:circulation}).

\newpage
\section{The limit $M \to \infty$: fixed-point polytope and circulations}\label{sec:worked-examples}

\subsection{Orientation}

Theorem~\ref{thm:drift} analysed the stochastic recursion for finite corpus size $M$.
In the limit $M \to \infty$, each generation samples perfectly from the model's distribution: no rare pattern is lost by chance, and the distributional recursion $\rho_{t+1} = G_n(R_n(\rho_t))$ becomes exact---the updated $n$-gram law is fully determined by the current one, even though the underlying text generation remains stochastic.
The fixed points of this map---the distributions that reproduce themselves exactly under one round of fit, generate, refit---are the natural candidates for long-run equilibria.
This section characterises the full set of such fixed points, first through small worked examples where every solution can be found by hand, then through a general theorem that connects the fixed-point set to circulations on de~Bruijn graphs.

\subsection{Bigrams over binary tokens}\label{sec:worked-binary-bigram}

Fix the token set $\Sigma = \{0, 1\}$ and model order $n = 2$.
The context space is $\mathcal{C} = \Sigma^1 = \{0, 1\}$, and a conditional distribution over the next token is specified by two parameters:
\[
\theta_0 := p(0 \mid 0), \qquad \theta_1 := p(0 \mid 1),
\]
with $p(1 \mid 0) = 1 - \theta_0$ and $p(1 \mid 1) = 1 - \theta_1$.

There are four bigrams $(00)$, $(01)$, $(10)$, $(11)$ and two unigrams $(0)$, $(1)$.
A bigram distribution is a vector $\rho = (a, b, c, d)$ on $\{00, 01, 10, 11\}$ with $a + b + c + d = 1$ and $a, b, c, d \ge 0$.

\paragraph{The fixed-point condition.}
A distribution $\rho$ is a \emph{fixed point} if it reproduces itself under one round of the recursion: fit the conditional probabilities $p(a \mid c)$ from $\rho$, generate an infinite text from those conditionals, and measure the resulting bigram frequencies---if they equal $\rho$, it is a fixed point.
Concretely, the map $\rho \mapsto G_2(R_2(\rho))$ extracts the conditional probabilities from $\rho$, computes the long-run frequencies of the chain they define, and returns the corresponding bigram distribution.
The fixed-point equation is $\rho = G_2(R_2(\rho))$.

\paragraph{Extracting the conditional probabilities.}
Given $\rho = (a, b, c, d)$, the induced conditional probabilities are
\[
\theta_0 = \frac{a}{a+b}, \qquad \theta_1 = \frac{c}{c+d},
\]
so that each conditional distribution is obtained by normalising the corresponding bigram counts.

\paragraph{Stationary distribution.}
The context chain has transition matrix
\[
K = \begin{pmatrix} \theta_0 & 1-\theta_0 \\ \theta_1 & 1-\theta_1 \end{pmatrix},
\]
with stationary distribution
\[
\pi(0) = \frac{\theta_1}{1 - \theta_0 + \theta_1}, \qquad \pi(1) = \frac{1 - \theta_0}{1 - \theta_0 + \theta_1}.
\]

\paragraph{Generated bigram law.}
The rollout law is $\rho'(c, a) = \pi(c)\, p(a \mid c)$.
Substituting and simplifying:
\begin{align*}
\rho'(00) &= \frac{ac}{ac + 2bc + bd}, \qquad
\rho'(01) = \frac{bc}{ac + 2bc + bd}, \\[4pt]
\rho'(10) &= \frac{bc}{ac + 2bc + bd}, \qquad
\rho'(11) = \frac{bd}{ac + 2bc + bd}.
\end{align*}
The key observation is that $\rho'(01) = \rho'(10)$ regardless of the values of $a$, $b$, $c$, $d$.

\paragraph{Solving $\rho = \rho'$.}
Comparing $\rho'(01) = \rho'(10)$ with the requirement that $\rho(01) = b$ and $\rho(10) = c$, the fixed-point equation forces $b = c$.
Once $b = c$, the denominator becomes $ac + 2b^2 + bd = b(a + 2b + d) = b \cdot 1 = b$ (using $a + 2b + d = 1$), and each component $\rho'(ij)$ reduces to $\rho(ij)$ identically.
No further constraints arise.

\begin{proposition}[Fixed-point set for binary bigrams]\label{prop:binary-bigram-fp}
The set of all fixed points of the descriptive bigram recursion over $\Sigma = \{0, 1\}$ is
\[
\mathcal{F} = \bigl\{\, (a,\, b,\, b,\, d) \in \mathbb{R}^4 : a + 2b + d = 1,\; a, b, d \ge 0 \,\bigr\}.
\]
This is a $2$-simplex (triangle) inside the $3$-simplex $\Delta^3$ of all bigram distributions.
The single defining constraint $\rho(01) = \rho(10)$ is the \emph{stationary flow-balance} condition: in any stationary Markov chain on $\{0, 1\}$, the probability flux $0 \to 1$ equals the flux $1 \to 0$.
\end{proposition}

\begin{proof}
We showed above that $\rho = G_2(R_2(\rho))$ if and only if $b = c$.
The set of $(a, b, d)$ satisfying $a + 2b + d = 1$ with $a, b, d \ge 0$ is a $2$-simplex.
\end{proof}

\paragraph{Convexity.}
$\mathcal{F}$ is convex: it is the intersection of the linear subspace $\{b = c\}$ with the simplex $\Delta^3$, and the intersection of convex sets is convex.
Explicitly, if $\rho_1, \rho_2 \in \mathcal{F}$ and $\lambda \in [0,1]$, then $\lambda \rho_1 + (1-\lambda)\rho_2$ still satisfies $b = c$, still sums to $1$, and has all entries non-negative.

\paragraph{Extreme points.}
The three vertices of $\mathcal{F}$ are:
\begin{center}
\begin{tabular}{@{}lccc@{}}
\toprule
\textbf{Vertex} & $\boldsymbol{\rho}$ & \textbf{Kernel} & \textbf{Description} \\
\midrule
$V_1$ & $(1,\, 0,\, 0,\, 0)$ & $p(0 \mid 0) = 1$ & absorbing at token $0$: $\ldots 0000 \ldots$ \\
$V_2$ & $(0,\, 0,\, 0,\, 1)$ & $p(1 \mid 1) = 1$ & absorbing at token $1$: $\ldots 1111 \ldots$ \\
$V_3$ & $(0,\, \tfrac{1}{2},\, \tfrac{1}{2},\, 0)$ & $p(1 \mid 0) = 1,\; p(0 \mid 1) = 1$ & deterministic alternation: $\ldots 0101 \ldots$ \\
\bottomrule
\end{tabular}
\end{center}
These are exactly the three \emph{deterministic ergodic} Markov chains on $\{0, 1\}$---the ones where every $p(\cdot \mid c)$ is a point mass and the chain has a unique stationary distribution.

\begin{remark}[The fourth deterministic rule]
There is a fourth rule in which each token always repeats itself: $p(0 \mid 0) = 1$, $p(1 \mid 1) = 1$.
Under this rule, a text that starts with $0$ stays at $0$ forever, and likewise for $1$; the long-run frequency of $0$ depends entirely on how the text began.
Any starting mixture $\rho = (\lambda,\, 0,\, 0,\, 1-\lambda)$ with $\lambda \in [0,1]$ is therefore self-consistent, and these mixtures trace out the edge $V_1$--$V_2$ of the triangle $\mathcal{F}$.
So this rule contributes an entire edge of fixed points rather than a single vertex.
\end{remark}

\paragraph{Induced unigram probabilities.}
At a fixed point $(a, b, b, d)$, the unigram probabilities are
\[
P(0) = a + b, \qquad P(1) = b + d.
\]
These range continuously from $P(0) = 0$ (vertex $V_2$) through $P(0) = \tfrac{1}{2}$ (vertex $V_3$ and many interior points) to $P(0) = 1$ (vertex $V_1$).

\paragraph{Alternative parametrisation by $(\theta_0, \theta_1)$.}
Every interior fixed point corresponds bijectively to a kernel $(\theta_0, \theta_1) \in (0,1)^2$:
\begin{align*}
a &= \frac{\theta_0\,\theta_1}{1-\theta_0+\theta_1}, \qquad
b = c = \frac{\theta_1(1-\theta_0)}{1-\theta_0+\theta_1}, \qquad
d = \frac{(1-\theta_0)(1-\theta_1)}{1-\theta_0+\theta_1}.
\end{align*}
The diagonal $\theta_0 = \theta_1$ gives the \emph{memoryless} (i.i.d.)\ chains, where the next token is independent of the current context and $P(0) = \theta_0 = \theta_1$.
The centre of the simplex $(\tfrac{1}{4}, \tfrac{1}{4}, \tfrac{1}{4}, \tfrac{1}{4})$ is the fair-coin case $\theta_0 = \theta_1 = \tfrac{1}{2}$.

\paragraph{Information-theoretic quantities.}
At a fixed point with kernel $(\theta_0, \theta_1)$, the entropy rate of the chain is
\[
h = \pi(0)\, H(\theta_0) + \pi(1)\, H(\theta_1),
\]
where $H(x) = -x \log_2 x - (1-x) \log_2(1-x)$ is the binary entropy function.
The entropy rate satisfies $0 \le h \le H(P(0)) \le 1$ bit, with $h = H(P(0))$ if and only if the chain is memoryless ($\theta_0 = \theta_1$), and $h = 0$ at all three vertices (where the chain is deterministic).
The maximum $h = 1$ bit is achieved uniquely at the fair-coin point $\theta_0 = \theta_1 = \tfrac{1}{2}$.

\paragraph{Connection to drift (Theorem~\ref{thm:drift}).}
Every interior fixed point is a \emph{transient} state of the neutral drift process: under Theorem~\ref{thm:drift}, finite resampling eventually drives the system to one of the absorbing vertices $V_1$ or $V_2$.
The alternating vertex $V_3$ is absorbing for the $n$-gram table (every context has a unique continuation, so resampling reproduces the same table) but is an unstable equilibrium in the sense that any perturbation reintroduces back-off and allows drift to act.
The fixed-point set $\mathcal{F}$ therefore describes the instantaneous self-consistency condition, not the long-run fate under finite-corpus recursion.

\subsection{Trigrams over binary tokens}\label{sec:worked-binary-trigram}

Increasing the model order to $n = 3$ moves the context space to $\mathcal{C} = \Sigma^2 = \{00, 01, 10, 11\}$ and the kernel to $p(c \mid ab)$ for $a, b, c \in \{0, 1\}$.
There are now eight $3$-grams.
Write the distribution as
\[
\rho = \bigl(\rho(000),\, \rho(001),\, \rho(010),\, \rho(011),\, \rho(100),\, \rho(101),\, \rho(110),\, \rho(111)\bigr).
\]

\paragraph{Flow-balance conditions.}
The fixed-point condition $\rho = G_3(R_3(\rho))$ again reduces to stationarity of the context distribution: the bigram marginal on the first two positions of a $3$-gram must equal the marginal on the last two.
Define the left and right marginals:
\[
L(ab) := \sum_c \rho(abc), \qquad R(ab) := \sum_x \rho(xab).
\]
Setting $L = R$ for each of the four bigram contexts yields four equations, of which three are independent:
\begin{alignat}{3}
\text{(I)}\quad  & \rho(001) &&= \rho(100), \notag \\
\text{(II)}\quad & \rho(011) &&= \rho(110), \label{eq:trigram-flow}\\
\text{(III)}\quad& \rho(101) &&= \rho(010) + \rho(011) - \rho(001). \notag
\end{alignat}
Constraint~(I) equates the flow out of context $00$ into token~$1$ with the flow into context $00$ from token~$1$; constraint~(II) does the same for context~$11$; and (III) is the corresponding balance for contexts $01$ and $10$.

\paragraph{The polytope.}
Using (I)--(III) to eliminate $\rho(100)$, $\rho(110)$, and $\rho(101)$, the free variables are $\rho_0 = \rho(000)$, $\rho_1 = \rho(001)$, $\rho_2 = \rho(010)$, $\rho_3 = \rho(011)$, $\rho_7 = \rho(111)$, subject to
\[
\rho_0 + \rho_1 + 2\rho_2 + 3\rho_3 + \rho_7 = 1, \qquad \rho_0, \rho_1, \rho_2, \rho_3, \rho_7 \ge 0, \qquad \rho_2 + \rho_3 \ge \rho_1.
\]
The last inequality ensures $\rho(101) \ge 0$.
This defines a $4$-dimensional convex polytope in $\Delta^7$.

\begin{proposition}[Extreme points for binary trigrams]\label{prop:binary-trigram-fp}
The fixed-point polytope $\mathcal{F}_3$ has exactly six extreme points.
Every extreme point is the trigram distribution of a periodic binary sequence in which a deterministic rule assigns to each context a unique next token.
The eight components of $\boldsymbol{\rho}$ give the frequencies of the eight trigrams $(000, 001, 010, 011, 100, 101, 110, 111)$:
\begin{center}
\renewcommand{\arraystretch}{1.15}
\begin{tabular}{@{}clll@{}}
\toprule
\textbf{Vertex} & \textbf{Sequence} & $\boldsymbol{\rho}$ & \textbf{Period} \\
\midrule
$V_1$ & $\ldots 000 \ldots$ & $(1,\, 0,\, 0,\, 0,\, 0,\, 0,\, 0,\, 0)$ & $1$ \\
$V_2$ & $\ldots 111 \ldots$ & $(0,\, 0,\, 0,\, 0,\, 0,\, 0,\, 0,\, 1)$ & $1$ \\
$V_3$ & $\ldots 010101 \ldots$ & $(0,\, 0,\, \tfrac{1}{2},\, 0,\, 0,\, \tfrac{1}{2},\, 0,\, 0)$ & $2$ \\
$V_4$ & $\ldots 001001 \ldots$ & $(0,\, \tfrac{1}{3},\, \tfrac{1}{3},\, 0,\, \tfrac{1}{3},\, 0,\, 0,\, 0)$ & $3$ \\
$V_5$ & $\ldots 011011 \ldots$ & $(0,\, 0,\, 0,\, \tfrac{1}{3},\, 0,\, \tfrac{1}{3},\, \tfrac{1}{3},\, 0)$ & $3$ \\
$V_6$ & $\ldots 00110011 \ldots$ & $(0,\, \tfrac{1}{4},\, 0,\, \tfrac{1}{4},\, \tfrac{1}{4},\, 0,\, \tfrac{1}{4},\, 0)$ & $4$ \\
\bottomrule
\end{tabular}
\end{center}
\end{proposition}

\begin{proof}
The polytope has six inequality constraints in four dimensions (after the normalisation equality).
A vertex occurs where exactly four constraints are tight.
Enumerating all $\binom{6}{4} = 15$ candidate subsets and solving yields exactly six feasible vertices, which are those listed.
Each is verified to satisfy $\rho = G_3(R_3(\rho))$ by direct computation.
That every vertex is deterministic follows from inspection: the non-zero entries of each $\rho$ are uniform on a single orbit of a deterministic context map.
\end{proof}

\paragraph{Structure of the extreme points.}
The six vertices are organised by the period of the corresponding binary sequence:
\begin{itemize}[leftmargin=*]
\item Period~$1$ ($2$ orbits): the constant sequences $\ldots 000 \ldots$ and $\ldots 111 \ldots$.
\item Period~$2$ ($1$ orbit): the alternation $\ldots 0101 \ldots$, with context cycle $01 \to 10 \to 01$.
\item Period~$3$ ($2$ orbits): $\ldots 001 \ldots$ with cycle $00 \to 01 \to 10 \to 00$, and its complement $\ldots 011 \ldots$ with cycle $01 \to 11 \to 10 \to 01$.
\item Period~$4$ ($1$ orbit): $\ldots 0011 \ldots$ with the full $4$-cycle $00 \to 01 \to 11 \to 10 \to 00$.
\end{itemize}
The pair $(V_1, V_2)$ and the pair $(V_4, V_5)$ are related by the bit-flip symmetry $0 \leftrightarrow 1$.
The vertices $V_3$ and $V_6$ are each unchanged by this swap.

\paragraph{Deterministic rules that do not visit all contexts.}
A deterministic rule assigns to each context exactly one successor.
There are $2^4 = 16$ such rules on the four contexts $\{00, 01, 10, 11\}$.
Of these, exactly $6$ eventually visit every context (these are the six vertices above) and $10$ do not---they get trapped in a subset of contexts.
A trapped rule has no single well-defined long-run distribution; instead, its long-run behaviour depends on the starting context.
As in the bigram case (Remark~2), each such rule contributes an entire edge or face of fixed points to the trigram fixed-point polytope $\mathcal{F}_3$, rather than a single vertex.

\subsection{$4$-grams over binary tokens}\label{sec:worked-binary-4gram}

At order $n = 4$ the context space is $\mathcal{C} = \Sigma^3 = \{000, 001, \ldots, 111\}$ with $8$ contexts, and there are $16$ possible $4$-grams.
The flow-balance conditions $L(c) = R(c)$ contribute $7$ independent linear constraints, giving an $8$-dimensional convex polytope $\mathcal{F}_4$ with exactly $19$ extreme points.
Each extreme point corresponds to a simple directed cycle in the de~Bruijn graph---a periodic sequence that visits some subset of contexts in a fixed cyclic order, with each visited context having a unique successor.
The $19$ cycles range from period~$1$ (the constant sequences, visiting a single context) to period~$8$ (visiting all eight contexts):

\begin{center}
\renewcommand{\arraystretch}{1.10}
\scriptsize
\begin{tabular}{@{}cll@{}}
\toprule
\textbf{Period} & \textbf{Sequence(s)} & \textbf{Count} \\
\midrule
$1$ & $\ldots 0 \ldots\,$, $\;\ldots 1 \ldots$ & $2$ \\
$2$ & $\ldots 01 \ldots$ & $1$ \\
$3$ & $\ldots 001 \ldots\,$, $\;\ldots 011 \ldots$ & $2$ \\
$4$ & $\ldots 0001 \ldots\,$, $\;\ldots 0011 \ldots\,$, $\;\ldots 0111 \ldots$ & $3$ \\
$5$ & $\ldots 00011 \ldots\,$, $\;\ldots 00111 \ldots$ & $2$ \\
$6$ & $\ldots 000111 \ldots\,$, $\;\ldots 001011 \ldots\,$, $\;\ldots 001101 \ldots$ & $3$ \\
$7$ & $\ldots 0001011 \ldots\,$, $\;\ldots 0001101 \ldots\,$, $\;\ldots 0010111 \ldots\,$, $\;\ldots 0011101 \ldots$ & $4$ \\
$8$ & $\ldots 00010111 \ldots\,$, $\;\ldots 00011101 \ldots$ & $2$ \\
\midrule
& \textbf{Total} & $\mathbf{19}$ \\
\bottomrule
\end{tabular}
\end{center}
The two period-$8$ orbits are the only ones long enough to visit all eight contexts; they are known as \emph{de Bruijn sequences} of order~$3$.
Most extreme points visit only a subset of contexts: the period-$1$ orbits visit just one, the period-$3$ orbits visit three, and so on.
For each extreme point of period $p$, the $4$-gram distribution is uniform on the $p$ visited $4$-grams, with each receiving mass $1/p$.

\paragraph{Symmetry.}
The bit-flip symmetry $0 \leftrightarrow 1$ pairs $8$ of the $19$ vertices into symmetric pairs and leaves $3$ unchanged (periods $2$, $4$, and $6$).
Every pair consists of a ``$0$-heavy'' and a ``$1$-heavy'' periodic sequence.

\subsection{The general pattern: simple cycles in the de Bruijn graph}

The three worked examples reveal a uniform structure.

\begin{proposition}[Extreme points of the $n$-gram fixed-point polytope]\label{prop:extreme-general}
For order-$n$ models over a token set of size $s$, the extreme points of the fixed-point polytope $\mathcal{F}_n$ are in one-to-one correspondence with the distinct simple cycles in the de Bruijn graph $B(n{-}1, s)$.
Each extreme point is the uniform distribution on the $n$-grams traversed by the cycle; each visited $n$-gram receives mass $1/p$ where $p$ is the cycle length.
\end{proposition}

Equivalently, the extreme points are the stationary distributions of the deterministic ergodic Markov chains on the context space $\Sigma^{n-1}$: each such chain follows a single periodic orbit, visiting $p$ distinct contexts in a fixed cyclic order.

\paragraph{Counting extreme points.}
Over binary tokens, the number of extreme points for each $n$ is the number of distinct simple cycles (up to rotation) in $B(n{-}1, 2)$:

\begin{center}
\begin{tabular}{@{}cccl@{}}
\toprule
$n$ & $|\mathcal{C}| = 2^{n-1}$ & $\dim \mathcal{F}_n$ & \# extreme points \\
\midrule
$2$ & $2$ & $2$ & $3$ \\
$3$ & $4$ & $4$ & $6$ \\
$4$ & $8$ & $8$ & $19$ \\
$5$ & $16$ & $16$ & $179$ \\
\bottomrule
\end{tabular}
\end{center}

The dimension of $\mathcal{F}_n$ is always $2^n - 2^{n-1} = 2^{n-1}$: the $n$-gram simplex has dimension $2^n - 1$, and the flow-balance conditions contribute $2^{n-1} - 1$ independent constraints.
The number of extreme points grows rapidly with $n$ and does not equal the total number of distinct binary periodic sequences (up to rotation) of period $\le 2^{n-1}$: the de Bruijn-graph constraint that all consecutive $(n{-}1)$-grams must be distinct is strictly stronger than primitivity of the periodic sequence.

\paragraph{Why not every necklace is valid.}
A \emph{necklace} is an equivalence class of binary strings under cyclic rotation---in other words, a periodic sequence considered up to where one starts reading it.
A binary necklace of period $p$ with $p \le 2^{n-1}$ may repeat an $(n{-}1)$-gram, making it impossible to assign a unique next token to that context.
For example, at $n = 4$ the necklace $00101$ (period~$5$) produces $3$-grams $001, 010, 101, 010, 100$, where $010$ appears twice with different successors ($1$ and $0$), so no deterministic rule can generate this sequence.
The valid necklace $00011$ produces $3$-grams $000, 001, 011, 110, 100$---all distinct.

\subsection{The general theorem: circulations on de Bruijn graphs}\label{sec:worked-circulation}

The pattern observed in the binary examples holds for any number of tokens $s$ and model order $n$.
The key insight is that the flow-balance condition identifying the fixed-point set is \emph{exactly} the circulation condition on the de Bruijn graph, for which the extreme-point structure is a classical result in combinatorial optimisation (see, e.g., Schrijver~\cite{appx:Schrijver2003}).

\addtocounter{theorem}{-1}
\begin{theorem}[continued: fixed-point polytope as circulation polytope]\label{thm:circulation}
Let $\Sigma$ be a finite token set of size $s \ge 2$ and $n \ge 2$ the model order.
The set $\mathcal{F}_n$ of self-consistent $n$-gram distributions (fixed points of the fit--generate--refit recursion in the limit $M \to \infty$) satisfies:
\begin{enumerate}[label=(\alph*)]
\item \emph{(Dimension.)} $\mathcal{F}_n$ is a convex polytope of dimension $s^{n-1}(s-1)$.

\item \emph{(Defining constraints.)} $\mathcal{F}_n$ is defined by $s^{n-1} - 1$ independent flow-balance constraints within the $(s^n - 1)$-simplex $\Delta^{s^n - 1}$.

\item \emph{(Extreme points.)} The extreme points of $\mathcal{F}_n$ are in one-to-one correspondence with simple directed cycles in the de Bruijn graph $B(n{-}1, s)$.
The extreme point corresponding to a cycle of length $p$ is the uniform distribution on the $p$ traversed $n$-grams, each receiving mass $1/p$.

\item \emph{(Convex decomposition.)} Every self-consistent $n$-gram distribution is a convex combination of these deterministic periodic-orbit distributions.
\end{enumerate}
\end{theorem}

\begin{proof}\

\medskip
\noindent\textbf{Fixed-point condition = flow balance.}
The fixed-point condition $\rho = G_n(R_n(\rho))$---where $R_n$ fits the conditional probabilities from $\rho$ and $G_n$ generates the resulting $n$-gram distribution---is equivalent to the statement that for every $(n{-}1)$-gram context $c$, the total probability of $n$-grams that \emph{begin} with $c$ equals the total probability of $n$-grams that \emph{end} with $c$:
\[
L(c) := \sum_{a \in \Sigma} \rho(c \cdot a) = \sum_{x \in \Sigma} \rho(x \cdot c) =: R(c), \qquad \forall\, c \in \Sigma^{n-1}.
\]
Write $p(a \mid c) = \rho(c \cdot a)/L(c)$ for the fitted conditional distribution and $\pi(c) = L(c)$ for the context frequency.
If $L(c) = R(c)$, then $\pi$ is the long-run context frequency of the chain defined by $p$ (since $R(c) = \sum_{c'} \pi(c') p(a \mid c')$ summed over predecessors $c'$ that transition into $c$), and $\pi(c)\, p(a \mid c) = L(c) \cdot \rho(c \cdot a)/L(c) = \rho(c \cdot a)$, so regeneration reproduces $\rho$.
Conversely, if $\rho$ is a fixed point, then $\pi(c) = L(c)$ must be the long-run context frequency, which gives $L(c) = R(c)$.

This is exactly the conservation-of-flow condition at every node of the de~Bruijn graph $B(n{-}1, s)$, where nodes are $(n{-}1)$-gram contexts and each $n$-gram $g = c \cdot a$ is a directed edge from $c$ to $\sigma(c, a) = c_2 \cdots c_{n-1} a$.
Together with $\sum_g \rho(g) = 1$ and $\rho(g) \ge 0$, the constraints define the polytope of non-negative unit circulations on $B(n{-}1, s)$.

\medskip
\noindent\textbf{Part~(a): dimension.}
The $s^{n-1}$ flow-balance equations $L(c) = R(c)$ sum to the same value ($=1$) on both sides, so at most $s^{n-1} - 1$ are independent.
They are generically independent, giving dimension $(s^n - 1) - (s^{n-1} - 1) = s^{n-1}(s - 1)$.

\medskip
\noindent\textbf{Part~(b): distinct cycles give distinct extreme points.}
Two simple directed cycles are \emph{distinct} if they traverse different edge sets (identifying rotations of the same cycle).
If cycles $\gamma_1$ and $\gamma_2$ are distinct, some $n$-gram $g$ is traversed by one but not the other: then $f_{\gamma_1}(g) > 0$ but $f_{\gamma_2}(g) = 0$ (or vice versa), so $f_{\gamma_1} \ne f_{\gamma_2}$.

\medskip
\noindent\textbf{Part~(c): each cycle flow $f_\gamma$ is an extreme point.}
Let $\gamma$ be a simple cycle of length $p$ and suppose $f_\gamma = \lambda f_1 + (1-\lambda) f_2$ with $f_1, f_2 \in \mathcal{F}_n$ and $0 < \lambda < 1$.

\emph{Step~1: $f_1$ and $f_2$ are supported on $\gamma$.}
If edge $e$ is not in $\gamma$, then $f_\gamma(e) = 0$, so $\lambda f_1(e) + (1-\lambda)f_2(e) = 0$.
Since $\lambda > 0$, $(1-\lambda) > 0$, and both $f_1(e), f_2(e) \ge 0$, we must have $f_1(e) = f_2(e) = 0$.

\emph{Step~2: $f_1 = f_2 = f_\gamma$.}
The edges of $\gamma$ form a single simple directed cycle: every visited node has exactly one incoming edge in $\gamma$ and one outgoing edge in $\gamma$.
For any non-negative unit circulation $f$ supported on these edges, flow balance at each visited node forces $f(\text{in-edge at } v) = f(\text{out-edge at } v)$.
Since the edges form a single cycle, this propagates around the cycle and forces all edge flows to be equal: $f(e) = \delta$ for all $e \in \gamma$.
The unit-flow condition gives $p\delta = 1$, so $\delta = 1/p$.
Therefore $f_1 = f_2 = f_\gamma$, and the decomposition is trivial.

\medskip
\noindent\textbf{Part~(d): every extreme point is a cycle flow.}
Let $f$ be an extreme point of $\mathcal{F}_n$.
Consider the support graph $G_f = (V_f, E_f)$ with $E_f = \{e : f(e) > 0\}$.
Since $f$ satisfies flow balance, every node in $V_f$ has at least one incoming and one outgoing edge in $E_f$, so $G_f$ contains at least one directed cycle.

If $G_f$ is a single simple cycle, we are done by Step~2 above.

If $G_f$ contains more than one cycle, we derive a contradiction.
Let $\gamma_1$ be a simple directed cycle in $G_f$ that does not use all edges of $E_f$, and let $h = f_{\gamma_1}$ be its elementary cycle flow.
Since $f \ne h$ (there are edges in $E_f$ outside $\gamma_1$), the vector $d = f - h$ is a nonzero vector that satisfies flow balance at every node and whose entries sum to zero.
For sufficiently small $\varepsilon > 0$, both
\[
f_+ = f + \varepsilon\, d \quad \text{and} \quad f_- = f - \varepsilon\, d
\]
are non-negative (since $f(e) > 0$ on $E_f$ and $d(e) = 0$ outside $E_f$), balanced, and have unit total flow.
Then $f = \tfrac{1}{2}(f_+ + f_-)$ with $f_+ \ne f_-$, contradicting $f$ being an extreme point.

\medskip
Therefore every extreme point of $\mathcal{F}_n$ is an elementary cycle flow, and every self-consistent $n$-gram distribution is a convex combination of these cycle flows.
\end{proof}

\begin{corollary}[Cycle--distribution bijection]
Different simple directed cycles in $B(n{-}1, s)$ produce different $n$-gram distributions.
Conversely, every extreme fixed-point distribution determines the cycle uniquely: it is recovered as the subgraph on $\{g : \rho(g) > 0\}$, which has a unique cyclic ordering since every visited node has exactly one successor in the support.
\end{corollary}

\paragraph{Verification across token counts and model orders.}
Exhaustive enumeration confirms the theorem for small cases.
Table~\ref{tab:extreme-points} lists the number of extreme points (simple directed cycles in the de Bruijn graph) for all feasible combinations of the number of tokens $s$ and model order $n$.
A Jupyter notebook reproducing and extending this table is included in the GitHub repository.

\begin{table}[H]
\centering
\begin{tabular}{@{}cccccc@{}}
\toprule
tokens $s$ & order $n$ & $|\mathcal{C}| = s^{n-1}$ & $\dim \mathcal{F}_n$ & \# extreme points & \# Hamiltonian cycles \\
\midrule
$2$ & $2$ & $2$ & $2$ & $3$ & $1$ \\
$2$ & $3$ & $4$ & $4$ & $6$ & $1$ \\
$2$ & $4$ & $8$ & $8$ & $19$ & $2$ \\
$2$ & $5$ & $16$ & $16$ & $179$ & $16$ \\
$2$ & $6$ & $32$ & $32$ & $30{,}176$ & $2{,}048$ \\
\midrule
$3$ & $2$ & $3$ & $6$ & $8$ & $2$ \\
$3$ & $3$ & $9$ & $18$ & $148$ & $24$ \\
$3$ & $4$ & $27$ & $54$ & $3{,}382{,}522$ & $373{,}248$ \\
\midrule
$4$ & $2$ & $4$ & $12$ & $24$ & $6$ \\
$4$ & $3$ & $16$ & $48$ & $120{,}538$ & $20{,}736$ \\
$4$ & $4$ & $64$ & $192$ & ${>}\,10^{20}$\rlap{$^{\dagger}$} & $\approx 1.9 \times 10^{20}$ \\
\midrule
$5$ & $2$ & $5$ & $20$ & $89$ & $24$ \\
\midrule
$6$ & $2$ & $6$ & $30$ & $415$ & $120$ \\
\midrule
$7$ & $2$ & $7$ & $42$ & $2{,}372$ & $720$ \\
\midrule
$8$ & $2$ & $8$ & $56$ & $16{,}072$ & $5{,}040$ \\
\midrule
$9$ & $2$ & $9$ & $72$ & $125{,}673$ & $40{,}320$ \\
\midrule
$10$ & $2$ & $10$ & $90$ & $1{,}112{,}083$ & $362{,}880$ \\
\bottomrule
\end{tabular}
\caption{Number of extreme points of the fixed-point polytope $\mathcal{F}_n$ for various numbers of tokens $s$ and model orders $n$.  Each extreme point corresponds to a simple directed cycle in the de Bruijn graph $B(n{-}1, s)$.  The rightmost column gives the number of \emph{Hamiltonian} cycles (de~Bruijn sequences), computed from the closed-form formula $(s!)^{s^{n-2}}/s^{n-1}$ \cite{appx:vanAardenne1951}; these form a subset of all simple cycles and so provide a lower bound on the total count.  For reference, modern large language models typically use token vocabularies of $s \approx 32{,}000$--$128{,}000$ and context windows exceeding $10^5$ tokens; at such scales the number of extreme points is beyond any conceivable enumeration.\\[4pt]
$^{\dagger}$\,Lower bound from the Hamiltonian cycles alone.}
\label{tab:extreme-points}
\end{table}
For example, the ternary bigram case ($s = 3$, $n = 2$) has $8$ extreme points corresponding to simple cycles in $B(1, 3)$: three constant sequences ($\ldots 0 \ldots$, $\ldots 1 \ldots$, $\ldots 2 \ldots$), three period-$2$ alternations ($\ldots 01 \ldots$, $\ldots 02 \ldots$, $\ldots 12 \ldots$), and two period-$3$ full cycles ($\ldots 012 \ldots$, $\ldots 021 \ldots$).

\subsection{Section 2 in brief}

The fixed-point condition for $n$-gram distributions reduces to flow conservation on the de Bruijn graph $B(n{-}1, s)$---a linear condition---so the fixed-point set is a convex polytope of dimension $s^{n-1}(s-1)$.
By the correspondence between extreme circulations and simple cycles (Theorem~\ref{thm:circulation}), the extreme points are in bijection with simple directed cycles in $B(n{-}1, s)$; each is the uniform distribution on the $n$-grams traversed by a single deterministic periodic orbit.
Every self-consistent $n$-gram distribution is a convex combination of these deterministic extremes.
This characterisation is fully general: it holds for any finite token set and any model order, and connects the fixed-point structure of the paper's recursion to well-studied objects in combinatorial optimisation and symbolic dynamics.

\newpage
\section{From fixed points to shallow and deep text}\label{sec:shallow-deep}

\subsection{Orientation}

Section~\ref{sec:worked-examples} characterised the fixed points of the neutral recursion as circulations on de~Bruijn graphs.
This section asks what those fixed points look like when viewed at shorter or longer window lengths, and uses that viewpoint to define \emph{shallow} and \emph{deep} text distributions.
The key tool is the project--lift test: marginalise an $r$-gram distribution to order~$n$, then roll forward from the induced order-$n$ continuation law.
This isolates exactly what longer context adds.

\subsection{Induced $j$-gram distributions and the projection of extreme points}\label{sec:induced-jgrams}

An $n$-gram distribution $\rho$ on $\Sigma^n$ induces, for every $j \le n$, a $j$-gram distribution by \emph{marginalisation}: for each $j$-gram $u \in \Sigma^j$,
\[
\rho_j(u) \;=\; \sum_{v \in \Sigma^{n-j}} \rho(uv).
\]
Because this map is linear, it sends the $n$-gram circulation polytope $\mathcal{F}_n$ to a convex set of $j$-gram distributions, and preserves convex combinations: if $\rho = \sum_i \lambda_i \rho_{\gamma_i}$ is a mixture of cycle distributions, then $\rho_j = \sum_i \lambda_i (\rho_{\gamma_i})_j$ with the same weights~$\lambda_i$.

For $j > n$, an $n$-gram fixed point also determines a $j$-gram distribution, because the fixed point determines an order-$n$ continuation law whose rollout specifies the frequency of every $j$-gram.
Concretely, $\rho_j(x_1 \cdots x_j) = \rho(x_1 \cdots x_n) \prod_{i=n+1}^{j} p(x_i \mid x_{i-n+1} \cdots x_{i-1})$,
where $p(w \mid c) = \rho(cw)/\rho_{\!n-1}(c)$ are the conditional probabilities extracted from $\rho$.

\begin{proposition}[Extremality is preserved upward]\label{prop:extreme-upward}
Let $\gamma$ be a simple directed cycle of period~$p$ in $B(n{-}1,s)$, and let $\rho_\gamma$ be the corresponding extreme point of the $n$-gram circulation polytope $\mathcal{F}_n$.
Then for every $j \ge n$, the induced $j$-gram distribution~$(\rho_\gamma)_j$ is the uniform distribution on $p$ distinct $j$-grams and is an extreme point of the $j$-gram circulation polytope $\mathcal{F}_j$.
Moreover, the map $\gamma \mapsto (\rho_\gamma)_j$ is injective.
\end{proposition}

\begin{proof}
The cycle $\gamma$ corresponds to a periodic token sequence $x_1 x_2 \cdots x_p$ of \emph{minimal} period~$p$.
At window length~$n$, the $p$~consecutive $n$-grams are all distinct (this is the simple-cycle condition in $B(n{-}1,s)$).
Now take any $j \ge n$.
If two $j$-grams starting at positions $i$ and $k$ (with $0 \le i < k < p$) were identical,
their length-$n$ prefixes---the $n$-grams at positions $i$ and $k$---would also be identical, contradicting distinctness of the $n$-grams.
So all $p$~consecutive $j$-grams are distinct, each appearing once per period, giving the uniform distribution with mass~$1/p$ on each.
This is precisely the extreme point of $\mathcal{F}_j$ corresponding to the same periodic orbit viewed as a simple cycle in $B(j{-}1,s)$.

For injectivity: two distinct simple cycles in $B(n{-}1,s)$ differ in at least one $n$-gram, hence their induced $j$-gram distributions differ as well (the $n$-gram marginal is determined by the $j$-gram distribution).
\end{proof}

\paragraph{Extremality is not preserved downward.}
The converse fails: an extreme point of $\mathcal{F}_n$ can marginalise to a non-extreme point of $\mathcal{F}_j$ for $j < n$.
For a concrete example, take $s = 3$, $n = 3$ and consider the trigram distribution $\rho$ that places mass $1/6$ on each of the six trigrams $001, 010, 102, 021, 210, 100$ and zero on the remaining $21$ trigrams.
These six trigrams form a simple cycle in $B(2,3)$, so $\rho$ is an extreme point of~$\mathcal{F}_3$.
But marginalising to bigrams, the bigram $(1,0)$ arises from two trigrams ($102$ and $100$), receiving mass $2/6 = 1/3$, while the remaining four bigrams each receive mass~$1/6$.
This non-uniform distribution is not the uniform distribution on any simple cycle in $B(1,3)$, hence it is not an extreme point of~$\mathcal{F}_2$.
Among the $148$ extreme points of $\mathcal{F}_3$ for $s = 3$, exactly $36$ lose extremality when projected to bigrams.

\paragraph{Why this asymmetry matters in practice.}
The upward--downward contrast is not merely a mathematical curiosity.
In a real text ecosystem, agents maintain a corpus by reading and publishing with some finite context window.
The asymmetry says that when the context window is \emph{increased}---for example, when a new generation of models can attend to longer passages---the statistical structure already present in the corpus is preserved: the longer-window agent can faithfully maintain the existing text statistics while potentially capturing additional structure.
When the context window is \emph{reduced}, however, the finer-grained statistics of the corpus generally cannot be maintained: a shorter-window agent loses access to distinctions that required the longer context, and the corpus drifts toward a coarser equilibrium.
In short, upgrading context is safe; downgrading context is lossy.

\subsection{Shallow and deep token distributions}\label{sec:shallow-text}

The concepts of shallow and deep are properties of probability distributions over token windows, not of individual texts.
In practice, a trained language model implicitly defines such a distribution: for any window length~$r$, the model assigns a probability to every $r$-token block.
A specific neural architecture (such as a transformer) is a proxy for this distribution; the definitions below apply to the distribution itself, regardless of how it is represented.

\begin{definition}[$n$-shallow and $n$-deep]\label{def:n-shallow}
Let $\mathcal{D}$ be a probability distribution over token sequences from a fixed vocabulary of $s$ tokens.
For any window length~$r$, write $\mathcal{D}_r$ for the induced distribution over contiguous $r$-token blocks, and write $\mathcal{D}_n$ for the induced distribution over $n$-token blocks.

The distribution~$\mathcal{D}$ is \emph{$n$-shallow} if for every $r \ge n$ the $r$-token distribution~$\mathcal{D}_r$ can be recovered from $\mathcal{D}_n$ alone: marginalising $\mathcal{D}_r$ to $n$-token blocks and rolling forward via the order-$(n{-}1)$ continuation law extracted from $\mathcal{D}_n$ reproduces $\mathcal{D}_r$ exactly.
Equivalently, $\mathcal{D}$ is $n$-shallow if and only if the next token depends only on the preceding $n{-}1$ tokens.

The distribution~$\mathcal{D}$ is \emph{$n$-deep} if it is not $(n{-}1)$-shallow: the distribution over $n$-token (or longer) blocks cannot be recovered from any window shorter than~$n$.
\end{definition}

An $n$-shallow distribution has no predictive structure beyond what a window of $n$ tokens can see.
Any agent with a longer context window---whether an $n$-gram model with lookahead or a transformer with a larger attention span---extracts no benefit from the additional context.
Conversely, an $n$-deep distribution contains genuine structure that shorter windows miss: a longer context window gives the agent a real advantage.

These are properties of the distribution, not of any particular sample or model architecture.
In the population-limit setting considered here, any sufficiently expressive next-token learner trained on an $n$-shallow environment has no further predictive gain from using a context window longer than $n{-}1$.
A learner trained on an $n$-deep corpus can, in principle, exploit the full depth.

\paragraph{Application to $n$-grams.}
In the $n$-gram setting, the distribution $\mathcal{D}_r$ is a distribution over $r$-grams, and the continuation law is the conditional $p(w \mid c) = \rho(cw)/\rho_{n-1}(c)$.
The \emph{project--lift test} makes $n$-shallowness checkable: given an $r$-gram distribution $\rho_r$ (with $r > n$), marginalise to $\rho_n$ and then \emph{lift}---extract the order-$(n{-}1)$ continuation law from $\rho_n$ and compute the $r$-gram distribution it implies.
Write $\mathrm{lift}_r(\rho_n)$ for the result.
Then $\rho_r$ is $n$-shallow if and only if $\mathrm{lift}_r(\rho_n) = \rho_r$.
When the equality fails, the gap $\rho_r - \mathrm{lift}_r(\rho_n)$ quantifies exactly the structure that lies beyond the $n$-gram window.

\paragraph{Example.}
Consider $s = 2$ and the $4$-gram distribution that places mass $1/3$ on each of $0010$, $0100$, $1001$ and zero elsewhere (the extreme point corresponding to the period-$3$ cycle in $B(3,2)$).

\begin{itemize}[leftmargin=*]
\item \emph{$3$-shallow}: marginalising to trigrams gives mass $1/3$ on each of $001, 010, 100$.  Lifting back via the trigram continuation law recovers the original $4$-gram distribution exactly, because every bigram context determines a unique next token.
\item \emph{Not $2$-shallow} (hence $3$-deep): marginalising to bigrams gives mass $1/3$ on each of $00, 01, 10$.  The bigram continuation law has $p(0 \mid 0) = 1/2$ and $p(1 \mid 0) = 1/2$, so lifting to $4$-grams produces $0000 \to 1/18$, $0010 \to 1/9$, and other $4$-grams that were not present in the original.  The bigram model spreads probability over $4$-grams that the true distribution excludes.
\end{itemize}
\paragraph{Measuring depth.}
The project--lift test gives a binary verdict ($n$-shallow or not), but the mismatch between the original distribution and the lifted version can also be \emph{measured}.
If $\mathcal{D}_r$ is the true $r$-token distribution and $\mathrm{lift}_r(\mathcal{D}_n)$ is the rollout from its induced order-$n$ continuation law, the KL divergence
\[
\KL\bigl(\mathcal{D}_r \,\big\|\, \mathrm{lift}_r(\mathcal{D}_n)\bigr)
\;=\;
H_{\times}\bigl(\mathcal{D}_r,\, \mathrm{lift}_r(\mathcal{D}_n)\bigr)
\;-\;
H(\mathcal{D}_r)
\]
decomposes depth into the cross-entropy of the true distribution against the $n$-gram rollout minus the entropy of the true distribution itself.
This quantity is zero if and only if the distribution is $n$-shallow; when it is positive, it measures exactly how much structure lies beyond the $n$-gram window.
Section~\ref{sec:info-diagnostics} develops this into a full diagnostic framework and applies it to matched descriptive-versus-normative experiments.

\subsection{Section 3 in brief}

Projecting fixed points to shorter or longer window lengths reveals an important asymmetry: increasing context preserves structure, while decreasing context is generally lossy.
This leads naturally to the distinction between $n$-shallow and $n$-deep distributions.
The project--lift test provides an exact criterion: a corpus is $n$-shallow precisely when its $r$-gram distribution is recovered by rolling forward from its induced order-$n$ continuation law.

\newpage
\section{Selection, publication rules, and Theorems 2--3}\label{sec:selection-theory}

\subsection{Orientation}

We now turn from properties of distributions to publication dynamics.
The question is not only what structure is present in a corpus, but whether selection preserves it, erases it, or creates an incentive for longer context.
Descriptive publication recycles visible text and drives the corpus toward $n$-shallowness.
Normative publication scores futures and can stabilise environments that need not be $n$-shallow.
Theorem~2 gives the fixed-point picture; Theorem~3 explains why later learners inherit the resulting public conditional.

\subsection{Descriptive versus normative publication}\label{sec:tut3-desc-norm}

Before introducing the formal selection machinery, we separate two fundamentally different kinds of publication rules, because the outcome of recursive publication depends entirely on which kind is in play.

\paragraph{Descriptive publication.}
In a \emph{descriptive} rule, agents may use any internal reasoning---including lookahead and chain-of-thought---but they do not apply external quality criteria (such as correctness checks, novelty requirements, or verification against outside sources) when deciding what to publish.
Some agents publish $r$-grams sampled directly from the current corpus; others generate new text from the induced order-$n$ continuation law, possibly after extensive internal deliberation.
What makes the rule \emph{descriptive} is not the absence of computation but the absence of normative standards: the agent accepts the statistical status quo and publishes accordingly.

Theorem~2 (stated below) shows that under descriptive publication, the corpus converges to an $n$-shallow distribution: the corpus $r$-gram distribution equals the rollout from its own induced order-$n$ continuation law, and lookahead becomes redundant.
Even agents that ``think'' before publishing cannot sustain depth beyond the $n$-gram window when no quality filter is applied.

\paragraph{Normative publication.}
In a \emph{normative} rule, agents score, filter, or verify their output against some quality criterion before publishing.
The criterion need not be external: it can be an outside check such as a unit test or experimental result, but equally an internal one such as deductive validity.
An agent that derives mathematical theorems from axioms is applying a normative rule---the proof must be valid---even though the criterion is entirely internal to the formal system.
The depth arises because the certificate (the proof) may need to be much longer than the agent's context window, so the published text carries structure that no short-context model can reproduce.

Under normative publication, the fixed-point corpus need not be $n$-shallow: the acceptance filter can sustain statistical structure beyond what the agent's $n$-gram window captures, and lookahead remains beneficial.

\paragraph{Why this matters.}
The descriptive/normative distinction determines whether recursive publication compresses or preserves structure.
Descriptive selection is self-defeating: it drives the corpus toward shallowness, erasing the very structure that made lookahead useful.
Normative selection can be self-sustaining: it maintains deep structure that rewards continued lookahead.
In real text ecosystems, both kinds of publication coexist---some agents generate entirely new text while others recycle, paraphrase, or rank existing text---and the balance between descriptive and normative forces shapes the long-run corpus.

\subsection{Agents, acceptance, and the selection mechanism}\label{sec:tut3-acceptance}

Consider a population of agents that all read the same public corpus and publish text back into it. An \emph{ordinary agent} fits an $n$-gram continuation law to the current corpus and generates text token by token from that law. A \emph{lookahead agent} uses the same fitted continuation law but reweights each candidate token by the probability that its continuation will be \emph{accepted}---that is, that the generated trace will satisfy some criterion over the next $L$ steps. The ordinary agent simply imitates; the lookahead agent selects.
\paragraph{The acceptance-conditioned continuation law.}
Let $p(a \mid c)$ be the fitted continuation law and let $E$ be an acceptance event---a criterion on future traces. For context $c$ and candidate token $a$, define the \emph{acceptance value}
\[
w(a \mid c) := \Prob_p(E \mid C_0 = c,\, A_1 = a)
\]
and the \emph{total acceptance probability}
\[
V(c) := \sum_b p(b \mid c)\, w(b \mid c).
\]
The \emph{acceptance-conditioned continuation law} is
\begin{equation}\label{eq:tut3-acceptance}
q(a \mid c) := \frac{p(a \mid c)\, w(a \mid c)}{V(c)} = \Prob_p(A_1 = a \mid E,\, C_0 = c).
\end{equation}
The second equality is Bayes' rule: $q$ is the conditional distribution of the first token given that the continuation will be accepted.

\begin{remark}[Unfolding the acceptance value for survival]
When the acceptance event requires the trace to survive for $L$ further steps, the acceptance value is a finite-horizon survival probability on the context graph. Conditioning on $A_1=a$, the first token is fixed, so the remaining uncertainty begins after the updated context
\[
c_1 = \sigma(c,a).
\]
Equivalently,
\[
w(a \mid c)
=
\sum_{a_2,\ldots,a_L}
\left(\prod_{i=2}^{L} p(a_i \mid c_{i-1})\right)
\left(\prod_{i=1}^{L} \1[\text{no failure at step } i]\right),
\]
where
\[
c_i = \sigma(c_{i-1}, a_i)\qquad (i \ge 2).
\]
In the finite-state setting, these values can be computed by dynamic programming through the survival probabilities $h_\ell(c)$.
\end{remark}

\paragraph{What counts as acceptance?}
The framework is deliberately general. In the Conan Doyle trigram illustration, acceptance means that the generated trajectory survives at full trigram order for at least $L$ more steps without being forced to back off, that is, without reverting from full-order trigram generation to a shorter-context rule. In verifier-rich domains, acceptance can mean that a unit test passes, a proof checker accepts, or a constraint is satisfied. Equation~\eqref{eq:tut3-acceptance} is a target conditional distribution, not a claim that any agent literally enumerates every $L$-step future: in an LLM system, the same effective bias can arise from chain-of-thought reasoning, sampling several candidate traces and discarding failures, or tool-assisted verification.

\begin{example}[Survival-based acceptance in the trigram illustration]
With $n = 3$, define the \emph{alive set} $\mathcal{C}_{\mathrm{alive}}$ as the set of bigram contexts with at least one outgoing sampled trigram. The acceptance event is $L$-step survival:
\[
E_L = \{\tau_{\mathcal{C}_{\mathrm{alive}}} > L\},
\]
where $\tau$ is the first time the generator is forced to back off from trigram generation to a shorter-context rule. Then
\[
w(a \mid c) = \1[\sigma(c,a) \in \mathcal{C}_{\mathrm{alive}}]\, h_{L-1}(\sigma(c,a)),
\]
where
\[
h_\ell(c) = \Prob_p(\tau > \ell \mid C_0 = c)
\]
is the $\ell$-step survival probability, computable by dynamic programming on the finite support graph.
\end{example}

\subsection{Why hard rules are analytically fragile}\label{sec:tut3-hard-soft}

At the cross-entropy baseline ($T=1$) used throughout this paper, the main analytical complication is back-off rather than selection sharpness.
But for lower-temperature or argmax ($T=0$) publication rules, a hard threshold makes the update map non-differentiable or even discontinuous: an arbitrarily small change in the fitted continuation law may switch which continuation wins and hence abruptly change what enters the corpus.
Smoothing the selection rule becomes essential for convergence analysis.

A soft publication rule replaces exact exclusion by continuous reweighting.
Near-miss continuations still retain positive mass, but with lower weight than better ones.
This makes the induced publication law vary smoothly with the fitted continuation law and is therefore much easier to analyse.
In favourable cases the resulting update operator is even contractive, yielding uniqueness and convergence of the fixed point.

One convenient family is Gibbs or Boltzmann reweighting.
Given a utility $U(p,c,x)$ assigned to a candidate continuation $x$ from context $c$, define
\[
\Pi_{\beta,p}(x \mid c)
:=
\frac{\exp(\beta\, U(p,c,x))}
{\sum_y \exp(\beta\, U(p,c,y))}.
\]
Here $\beta>0$ controls the sharpness of selection: small $\beta$ gives a permissive rule, while large $\beta$ approaches winner-take-all behaviour.
In this sense, hard publication is the singular limit of soft publication.

The same principle appears in probabilistic modelling: conditioning on an event that can have zero or near-zero probability is often unstable, whereas continuous likelihood reweighting is much better behaved.

\subsection{Theorem 2: fixed points under selection}\label{sec:tut3-theorem2}

The key question is whether repeated publication can settle to a stable distribution.

\begin{theorem}[Fixed points under selection]\label{thm:tut3-fixedpoint}
Let a population of $n$-gram agents with $L$-step lookahead (equivalently, $r$-gram agents with $r = n + L$) publish into a shared corpus.

\emph{Part (a): Descriptive publication.}
A fixed-point corpus $r$-gram distribution $\rho^\star$ satisfies
\[
\rho^\star = G_r(R_n(\rho^\star)),
\]
where $R_n$ refits an order-$n$ continuation law from the corpus and $G_r$ generates the corresponding $r$-gram rollout.
Equivalently, the fixed points are exactly the $n$-shallow distributions inside the fixed-point polytope $\mathcal F_r$: the corpus $r$-gram distribution coincides with the rollout from its induced order-$n$ continuation law.
Descriptive selection therefore drives the corpus toward $n$-shallowness, so lookahead becomes redundant at equilibrium.

\emph{Part (b): Normative publication.}
For the soft normative rules analysed here---those whose quality standard demands structure that no order-$n$ continuation law can produce---the fixed-point corpus is not $n$-shallow: the corpus $r$-gram distribution retains genuine structure beyond the $n$-gram window, and the KL divergence between the corpus distribution and the rollout from its induced order-$n$ continuation law is strictly positive at the fixed point.
Moreover, the KL divergence is bounded above by $L\log_2 s$ bits, and this bound is optimal: it is attained by cyclic de~Bruijn constructions (Section~\ref{sec:extremal-projectlift-gap}).
Normative selection is self-sustaining: it maintains deep structure that rewards continued lookahead.
\end{theorem}

\paragraph{Worked example (Part~(a)): descriptive publication by bigram agents with one-step lookahead.}
We illustrate Part~(a) of Theorem~2 by expanding the binary trigram example from Section~\ref{sec:worked-binary-trigram}.
Consider bigram agents ($n=2$) with one-step lookahead ($L=1$, so $r=3$) under descriptive publication.
The polytope $\mathcal F_3$ is the $4$-dimensional trigram polytope with six extreme points $V_1,\dots,V_6$ described earlier, and the descriptive fixed-point condition becomes
\[
\rho^\ast = G_3(R_2(\rho^\ast)),
\]
so the trigram distribution must equal the trigram rollout from its induced bigram continuation law.

Away from the singular corner $(\theta_0,\theta_1)=(1,0)$, a binary bigram continuation law $p$ is parameterised by two probabilities,
\[
\theta_0=p(0\mid 0),\qquad \theta_1=p(0\mid 1),
\]
with the complementary transitions $p(1\mid 0)=1-\theta_0$ and $p(1\mid 1)=1-\theta_1$.
The unique stationary distribution on contexts is
\[
\pi(0)=\frac{\theta_1}{1-\theta_0+\theta_1},\qquad
\pi(1)=\frac{1-\theta_0}{1-\theta_0+\theta_1}.
\]
The corresponding trigram distribution is $\rho^\ast(abc)=\pi(a)\,p(b\mid a)\,p(c\mid b)$, giving the eight values
\[
\begin{aligned}
\rho(000) &= \pi(0)\,\theta_0^2, &\quad
\rho(001) &= \pi(0)\,\theta_0(1{-}\theta_0), \\
\rho(010) &= \pi(0)\,(1{-}\theta_0)\,\theta_1, &\quad
\rho(011) &= \pi(0)\,(1{-}\theta_0)(1{-}\theta_1), \\
\rho(100) &= \pi(1)\,\theta_1\,\theta_0, &\quad
\rho(101) &= \pi(1)\,\theta_1\,(1{-}\theta_0), \\
\rho(110) &= \pi(1)\,(1{-}\theta_1)\,\theta_1, &\quad
\rho(111) &= \pi(1)\,(1{-}\theta_1)^2.
\end{aligned}
\]
As $(\theta_0,\theta_1)$ ranges over $[0,1]^2$, this traces out a $2$-dimensional surface $\mathcal M_{3,2}$ inside the $4$-dimensional polytope $\mathcal F_3$.
The three vertices of $\mathcal F_3$ that lie on this surface are easily verified:
$(\theta_0,\theta_1)=(1,1)$ gives $\pi(0)=1$ and $\rho(000)=1$ (vertex~$V_1$);
$(0,0)$ gives $\pi(1)=1$ and $\rho(111)=1$ ($V_2$);
$(0,1)$ gives $\pi(0)=\pi(1)=\tfrac{1}{2}$ and $\rho(010)=\rho(101)=\tfrac{1}{2}$ ($V_3$, the alternating cycle).
The full set $\mathcal M_{3,2}$ is this parameterised surface together with its limiting edge $\mathrm{conv}\{V_1,V_2\}$ as $(\theta_0,\theta_1)\to(1,0)$.

Three features are worth recording:
\begin{enumerate}[leftmargin=*]
\item \emph{Dimension drop.} $\mathcal M_{3,2}$ is $2$-dimensional inside the $4$-dimensional polytope $\mathcal F_3$. More generally, $\dim \mathcal M_{r,n}=s^{n-1}(s-1)$, the same as $\dim \mathcal F_n$.
\item \emph{Non-convexity.} $\mathcal M_{3,2}$ is not convex: a mixture of two $2$-shallow distributions is generically not $2$-shallow, because the condition $\rho^\ast(abc)=\pi(a)\,p(b\mid a)\,p(c\mid b)$ imposes quadratic constraints on the trigram probabilities.
\item \emph{Extreme-point filter.} Of the six vertices $V_1,\dots,V_6$ of $\mathcal F_3$, only $V_1$ (all zeros), $V_2$ (all ones), and $V_3$ (the alternating cycle) lie in $\mathcal M_{3,2}$. The distributions corresponding to the period-$3$ and period-$4$ cycles have genuine trigram structure that cannot be generated by any bigram law.
\end{enumerate}

For example, for the cycle $001001\ldots$ one has $\rho(001)=1/3$, so $\pi(0)>0$.
Since $\rho(000)=\pi(0)p(0\mid 0)^2=0$, this forces $p(0\mid 0)=0$.
But then
\[
\rho(100)=\pi(1)p(0\mid 1)p(0\mid 0)=0,
\]
contradicting $\rho(100)=1/3$.

\paragraph{Worked example (Part~(b)): why normative publication can sustain depth.}
We now illustrate Part~(b) using the same setting ($s=2$, $n=2$, $r=3$).
The three vertices of $\mathcal F_3$ that are \emph{not} $2$-shallow are $V_4$, $V_5$, and $V_6$.
For each, we can compute the induced bigram law, the bigram rollout, and the KL divergence that measures the depth lost by the bigram window.
Writing trigram probabilities in the order $(000,\,001,\,010,\,011,\,100,\,101,\,110,\,111)$:

\medskip\noindent\emph{Vertex $V_4$ (cycle $001\ldots$, period~$3$).}
The distribution is $\rho = (0,\,\tfrac{1}{3},\,\tfrac{1}{3},\,0,\,\tfrac{1}{3},\,0,\,0,\,0)$.
The induced bigram law has $\theta_0 = \tfrac{1}{2}$, $\theta_1 = 1$, with stationary distribution $\pi(0)=\tfrac{2}{3}$, $\pi(1)=\tfrac{1}{3}$.
The bigram rollout is $\tilde\rho = (\tfrac{1}{6},\,\tfrac{1}{6},\,\tfrac{1}{3},\,0,\,\tfrac{1}{6},\,\tfrac{1}{6},\,0,\,0)$, which spreads probability to trigrams $000$ and $101$ that do not appear in the cycle.
The KL divergence is $D_{\KL}(\rho\,\|\,\tilde\rho) = \tfrac{2}{3}$ bits.

\medskip\noindent\emph{Vertex $V_5$ (cycle $011\ldots$, period~$3$).}
By the bit-flip symmetry $0\leftrightarrow 1$, the calculation mirrors $V_4$: $\theta_0 = 0$, $\theta_1 = \tfrac{1}{2}$, and $D_{\KL} = \tfrac{2}{3}$ bits.

\medskip\noindent\emph{Vertex $V_6$ (cycle $0011\ldots$, period~$4$).}
The distribution is $\rho = (0,\,\tfrac{1}{4},\,0,\,\tfrac{1}{4},\,\tfrac{1}{4},\,0,\,\tfrac{1}{4},\,0)$.
Here $\theta_0 = \theta_1 = \tfrac{1}{2}$, so the stationary distribution is $\pi(0) = \pi(1) = \tfrac{1}{2}$ and the bigram rollout is the \emph{uniform} distribution $\tilde\rho = (\tfrac{1}{8},\,\ldots,\,\tfrac{1}{8})$.
The KL divergence is $D_{\KL}(\rho\,\|\,\tilde\rho) = 4 \times \tfrac{1}{4}\log_2 2 = 1$ bit---the largest gap of any distribution in~$\mathcal F_3$.

\medskip
The pattern is clear: the longer the period of the cycle, the more trigram structure is invisible to the bigram window, and the larger the KL divergence.
A numerical search over the full polytope confirms that $V_6$ achieves the global maximum of $1$~bit.
Under descriptive publication, all three gaps would collapse to zero as the corpus is driven toward $\mathcal M_{3,2}$.
Under normative publication, an acceptance criterion that favours the relevant pattern can counteract this drift, sustaining trigram structure that the bigram window alone cannot generate and keeping the KL divergence positive.

\begin{remark}
The parameters in this example are small ($s=2$, $n=2$, $r=3$), but every ingredient---the fixed-point polytope, the $n$-shallow surface, the induced continuation law, and the KL divergence---is defined by finite-dimensional linear algebra and elementary probability.
For any specific distribution in $\mathcal F_r$, the KL divergence between the distribution and its order-$n$ rollout can be computed in closed form for general $s$, $n$, and $r$.
In particular, for specific normative publication rules with an explicit target, the depth of the resulting fixed point is an explicit function of the parameters, not merely a qualitative prediction.
\end{remark}

\subsection{From selection to inheritance}

The acceptance-conditioned continuation law matters because it is what later learners see in the environment.
Once the public corpus has been shaped by selection, an ordinary next-token learner does not need to re-perform the selection step: it trains on the filtered environment itself.
This is why Theorem~\ref{thm:tut3-cebridge} is the natural bridge beyond $n$-grams.
The inheritance claim is therefore about the target conditional written into the corpus, not about the internal mechanism that first produced it.

\subsection{Theorem 3: cross-entropy inheritance}\label{sec:tut3-theorem3}

\begin{theorem}[Cross-entropy inheritance]\label{thm:tut3-cebridge}
Let $q$ be a public next-token conditional generated by an environment published as a collection of texts.
Training a later learner by expected next-token cross-entropy recovers $q$ whenever the model class contains $q$.
More generally, when the class does not contain $q$, cross-entropy minimisation recovers the KL-closest conditional in that class.
\end{theorem}

This is the bridge beyond $n$-grams.
For an unconstrained order-$n$ $n$-gram table, the theorem is just relative-frequency estimation stated in cross-entropy language.
Fix a context $c$ and write $\hat p(a \mid c)=N(c,a)/N(c)$ for the empirical next-token frequencies in the published environment.
Among all conditional distributions on that context, $\hat p(\cdot \mid c)$ is exactly the one that minimises the empirical next-token cross-entropy.
So if the $n$-gram class is large enough to represent the target conditional, and the learner later generates by \emph{sampling} from its fitted table, then it selects the next token from exactly the same conditional distribution singled out by cross-entropy minimisation.
With finite data this is the empirical conditional rather than the population limit $q$, and greedy decoding is different again because it uses only $\arg\max_a \hat p(a \mid c)$ rather than the full distribution.

\paragraph{The architecture-independent message.}
The key consequence is that different architectures trained on the same filtered corpus converge to the same target conditional, to the extent permitted by their representational capacity and optimisation.
A smoothed trigram agent and a small neural network, both trained on an environment shaped by selection, move toward the same public conditional~$q$.
What is inherited is the public conditional; optimisation and approximation remain architecture-dependent.

\paragraph{Sampling versus argmax publication.}
The theorem is about the fitted conditional, not about how that conditional is later decoded into text.
If a learner recovers $q$ and then publishes by sampling, the published next-token law is $q$ itself.
If instead it publishes greedily, the induced policy is
\[
\pi_{\max}(a \mid c)=
\begin{cases}
1, & a \in \arg\max_u q(u \mid c)\ \text{with fixed tie-breaking},\\
0, & \text{otherwise.}
\end{cases}
\]
So a population of argmax agents does \emph{not} republish $q$; it republishes a deterministic policy derived from $q$.
The argmax map is discontinuous at ties and near-ties: a small change in $q$ can switch the winner and abruptly change what enters the corpus.
This is the same fragility discussed in Section~\ref{sec:tut3-hard-soft} for hard publication rules, and it means that multiple fixed points, path dependence, and abrupt winner switches are generic possibilities under greedy decoding.
Section~\ref{sec:heterogeneous} works out the geometry of the $T=0$ (argmax) fixed-point set in detail: in place of the convex circulation polytope that arises at $T=1$, one obtains a union of simplices built from vertex-disjoint cycles.

\subsection{Section 4 in brief}

This section introduced publication dynamics.
Descriptive publication (no quality standard applied) drives the corpus toward $n$-shallowness; normative publication (output scored or verified against some criterion) need not.
Theorem~2 makes this precise.
The acceptance-conditioned continuation law provides the formal selection mechanism; soft rules are analytically tractable while hard rules introduce discontinuities.
Theorem~3 (cross-entropy inheritance) completes the arc: later learners trained on the filtered corpus recover the public conditional, regardless of architecture.

\newpage
\section{Project--lift diagnostics for Theorem 2}\label{sec:info-diagnostics}

\subsection{Orientation}

Section~\ref{sec:shallow-deep} introduced the project--lift test: compare a corpus $r$-gram distribution with the rollout from its induced order-$n$ continuation law. This section applies that comparison dynamically along the recursion. The central distinction is between two questions that can come apart:
\begin{itemize}[leftmargin=*]
\item Does the recursion converge?
\item Does it converge to an $n$-shallow corpus?
\end{itemize}
The primary diagnostics are the $L^1$ and KL gaps between the corpus distribution and its project--lift rollout; entropy is a useful secondary contextual quantity.

\subsection{The project--lift diagnostics}

When the corpus distribution at generation $t$ is $\rho_t$ on length-$r$ blocks, the natural strong comparison object is the rollout from its induced order-$n$ continuation law:
\[
\widetilde{\rho}_t := G_r(R_n(\rho_t)).
\]
This gives two canonical diagnostics:
\[
D_t^{(1)} := \|\rho_t - \widetilde{\rho}_t\|_1,
\qquad
D_t^{(2)} := \KL(\rho_t \,\|\, \widetilde{\rho}_t).
\]

The $L^1$ gap measures total probability mismatch.
The KL divergence measures how surprising the corpus $r$-gram distribution would be if one insisted that it had been generated by the induced order-$n$ continuation law.
At a descriptive fixed point, both vanish.

Entropy is useful too, but it answers a different question.
The corpus $r$-gram entropy says how spread out the visible corpus is; the induced $n$-gram entropy says how spread out the fitted short-memory dynamics are.
Neither one, by itself, tells us whether the corpus $r$-gram distribution is actually order-$n$ generated.
For theorem testing, KL and $L^1$ are primary; entropy is secondary.

\subsection{Matched exact experiment: descriptive collapse versus normative plateau}

We illustrate the diagnostics on a deliberately matched pair of exact recursions, one descriptive and one normative, that differ only in the publication rule.
The experiment is designed to test whether the project--lift diagnostics can separate two situations that look superficially similar: both recursions converge, but only the descriptive one converges to an $n$-shallow corpus.

\paragraph{Construction.}
A seed order-$3$ continuation law $q_{\mathrm{seed}}$ is drawn by placing uniform mass on $25$ randomly chosen trigrams (out of $5^3 = 125$ possible).
The initial corpus $5$-gram distribution is the exact rollout $\rho_0 = G_5(q_{\mathrm{seed}})$, with $|\Sigma|=5$, $n=3$, $r=5$.
At each generation the entire corpus is replaced ($\alpha = 1$ in the notation of Section~\ref{sec:ngram-basics}): $80\%$ of the new material is contributed by the lookahead $r$-agents and $20\%$ by the order-$n$ agents who publish from the fitted continuation law.
In the descriptive case, the $r$-agents recycle the current corpus $5$-gram distribution.
In the normative case, they publish from a fixed target: an independent random $5$-gram distribution drawn from a Dirichlet prior over $30$ randomly chosen $5$-grams.
This target is genuinely $3$-deep---it was not generated by any order-$3$ continuation law---so a large persistent project--lift gap is structurally possible, and the normative run tests whether the diagnostics can detect it.
Both runs share the same initial corpus and the same random seed; only the publication rule differs.

(The figures in this section are generated by the exact script \path{build_theorem2_information_tutorial_case.py} in \path{GitHub/scripts/}.
The accompanying notebook \path{ngram_theorem2_strong_diagnostics_teaching.ipynb} in \path{GitHub/notebooks/basic_theory/} lets the reader experiment with different settings, including the vocabulary size, model order~$n$, block length~$r$, the lookahead mixing fraction, the seed initialisation mode (uniform, Dirichlet, or random text), and the choice between descriptive and normative publication rules.)

\paragraph{Results.}
By generation $40$, the descriptive run has driven the KL divergence below $10^{-5}$ and the $L^1$ gap below $3\times10^{-3}$.
The normative run, by contrast, has converged to a stable corpus distribution with KL divergence $2.57$ bits and $L^1$ gap $1.47$.
The two recursions therefore separate sharply: the descriptive diagnostics collapse, whereas the normative diagnostics stabilise at a nonzero plateau.
A KL gap of $2.57$ bits is not a small residual mismatch: on average, encoding a $5$-gram drawn from the normative equilibrium with the rollout from the induced trigram continuation law costs about $2.57$ extra bits per block.

\begin{figure}[t]
\centering
\includegraphics[width=\linewidth]{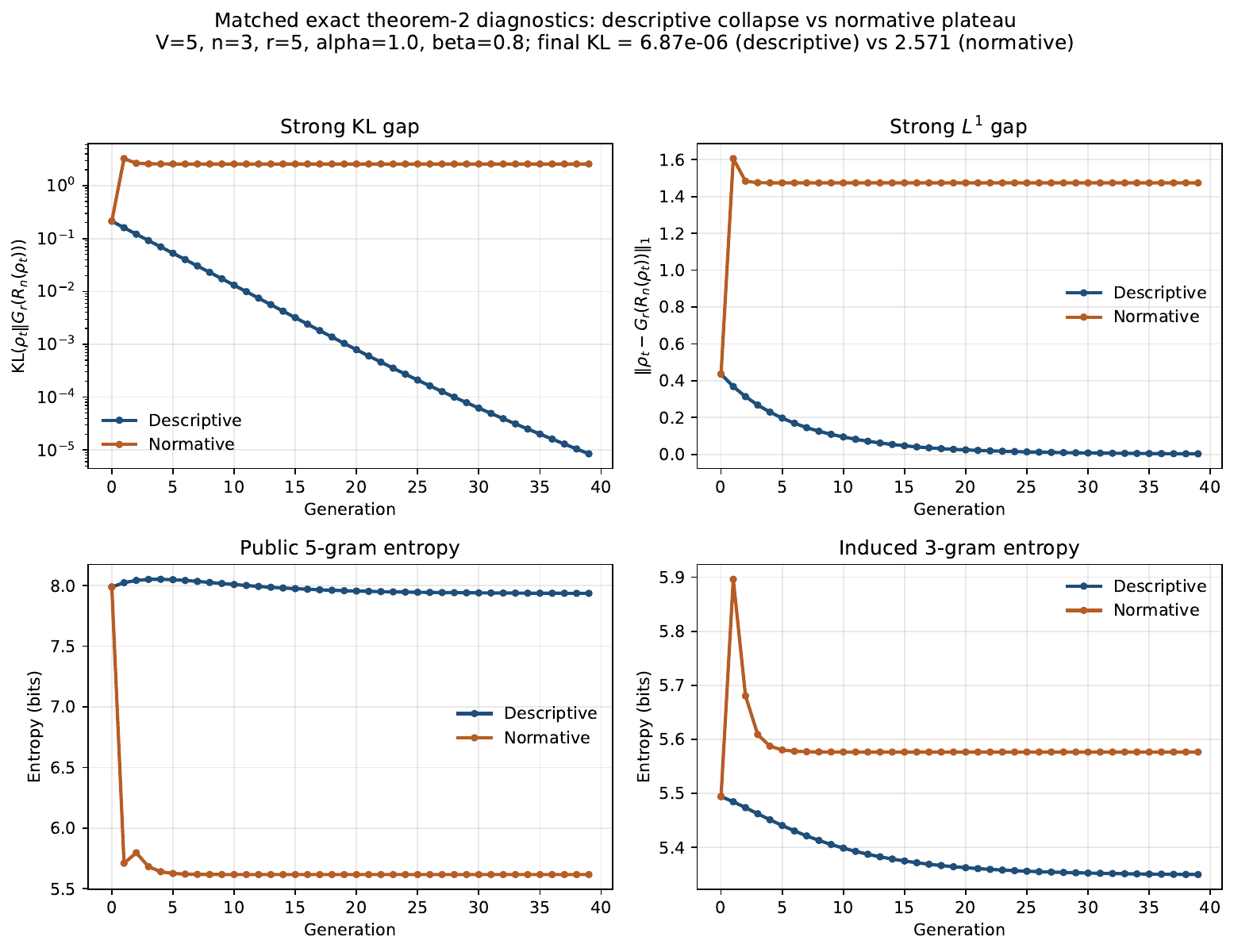}
\caption{\textbf{Matched exact theorem-2 diagnostics.}
The descriptive recursion drives both the KL divergence and the $L^1$ gap essentially to zero, while the normative recursion settles to a stable nonzero plateau.
The entropy panels show that the corpus distribution and induced trigram continuation law can both stabilise even when the KL divergence remains positive.}
\label{fig:tut4-info-diagnostics}
\end{figure}

\subsection{Convergence versus $n$-shallowness}

This is the right place to be careful with interpretation.
In the normative case, the issue is not failure of convergence.
The recursion does converge, but it converges to a corpus distribution that is not order-$n$ generated.
In the present matched run, the residual mismatch is also macroscopically large ($2.57$ bits of KL divergence), so this is not a near-shallow artefact but a genuine stable departure from $n$-shallowness.
This is exactly why convergence and $n$-shallowness are different questions: the normative recursion can stabilise without satisfying the descriptive fixed-point condition.

\subsection{Where the normative gap lives}

At the final normative equilibrium (recall that the tokens here are synthetic labels $0$--$4$, not real words), the largest single $5$-gram mismatch is on block \texttt{3-0-4-4-4}, whose probability differs by roughly $0.11$ between the corpus distribution and the rollout from the induced trigram continuation law.
Several other blocks, such as \texttt{1-3-2-2-4}, \texttt{4-0-3-3-3}, and \texttt{4-1-0-0-2}, retain comparably visible discrepancies.

Knowing that a gap exists is not enough; we want to know where it lives. Using the prefix/suffix decomposition from Section~\ref{sec:shallow-deep}, the total $L^1$ gap can be separated into a prefix-frequency term and a weighted suffix-conditional term. Any $5$-gram $w_1 w_2 w_3 w_4 w_5$ can be decomposed into its trigram prefix $w_1 w_2 w_3$ and the conditional suffix $w_4 w_5 \mid w_1 w_2 w_3$.

Because the normative target is genuinely $3$-deep, the gap now lives in both components: different prefix frequencies \emph{and} different conditional suffix laws given the same prefix.
Numerically, the prefix-frequency term contributes $1.06$ to the total $L^1$ gap, while the weighted suffix-conditional term contributes $1.11$.
This contrasts with a near-shallow target, where the suffix conditionals would already be aligned and only the prefix masses would differ.
Figure~\ref{fig:tut4-normative-gap} highlights the visible prefix contribution; the suffix-conditional component is also nonzero in this genuinely $3$-deep example.

\begin{figure}[t]
\centering
\includegraphics[width=\linewidth]{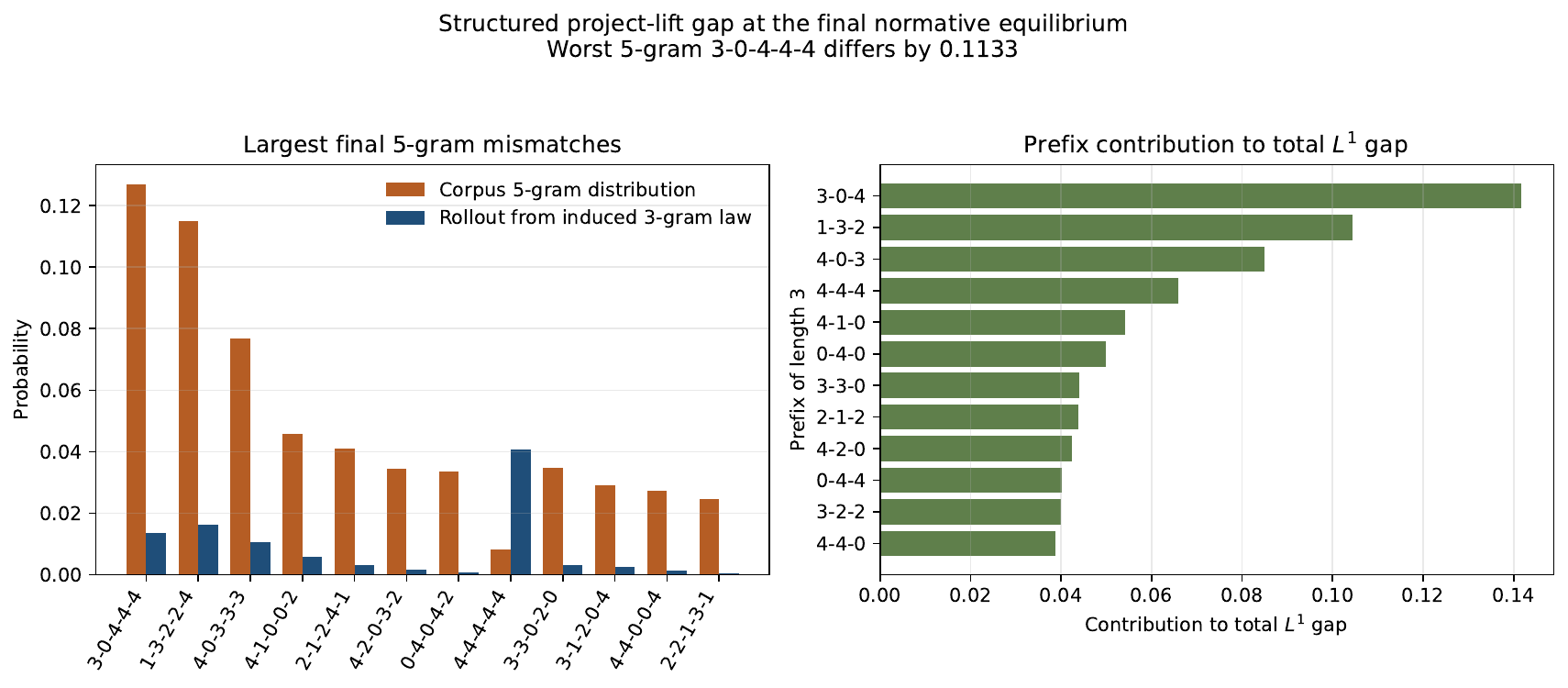}
\caption{\textbf{The stable normative gap has visible structure.}
Left: the largest final $5$-gram probability mismatches between the corpus distribution and the rollout from its induced trigram continuation law.
Right: the prefixes contributing most to the total $L^1$ gap.
The right panel visualises the prefix component only; in this genuinely $3$-deep target, the suffix-conditional component is also nonzero.}
\label{fig:tut4-normative-gap}
\end{figure}

This is why the project--lift diagnostics are useful.
They separate three questions that are easy to conflate:
\begin{itemize}[leftmargin=*]
\item Has the recursion converged?
\item Has the corpus $r$-gram distribution become order-$n$ generated?
\item If not, where is the residual mismatch located?
\end{itemize}

The matched exact run answers them cleanly.
In the descriptive case: yes, yes, nowhere substantial.
In the normative case: yes, no, and the residual mismatch lives in both prefix frequencies and suffix conditionals.

\subsection{An extremal project--lift KL gap}\label{sec:extremal-projectlift-gap}

The matched experiment above shows that a normative publication rule can converge to a stable corpus with
a substantial project--lift gap. It is natural to ask how large that gap can be in principle. The answer is
exact: for fixed alphabet size $s$ and hidden extra depth $L = r-n$, the block-level KL divergence is at most
$L \log_2 s$ bits, and this bound is sharp.

The sharpness witness is a cyclic de~Bruijn word. We first record why such a word exists.

\begin{lemma}[Eulerian de Bruijn cycle and induced cyclic word]\label{lem:debruijn-euler}
Let $\Sigma$ be an alphabet of size $s$, and let $B(n-1,s)$ be the de~Bruijn graph whose vertices are the
$(n-1)$-grams over $\Sigma$ and whose edges are the $n$-grams. Then:

\begin{enumerate}[leftmargin=*]
\item $B(n-1,s)$ has an Eulerian cycle of length $s^n$.
\item Any such Eulerian cycle induces a cyclic word $w$ of length $s^n$ whose cyclic length-$n$ windows are exactly the $s^n$ distinct $n$-grams over $\Sigma$, each appearing once.
\item The primitive period of $w$ is exactly $s^n$.
\end{enumerate}
\end{lemma}

\begin{proof}
Each vertex has indegree and outdegree equal to $s$, since one may prepend or append any symbol of $\Sigma$.
The graph is strongly connected: from any vertex $c=c_1\cdots c_{n-1}$ to any vertex $d=d_1\cdots d_{n-1}$,
append $d_1,\dots,d_{n-1}$ one by one. Hence the directed Euler criterion applies, so $B(n-1,s)$ has an Eulerian
cycle. Its length is the number of edges, namely $s^n$.

Write the Eulerian cycle as a cyclic sequence of edges $e_0,\dots,e_{s^n-1}$. Consecutive edges overlap in
their last and first $(n-1)$ symbols, so reading the new symbol contributed by each edge produces a cyclic word
$w$ of length $s^n$. Its cyclic length-$n$ windows are exactly the edge labels, hence every $n$-gram appears
exactly once.

If $w$ had period $d<s^n$, then there could be at most $d$ distinct cyclic length-$n$ windows. But $w$ has
exactly $s^n$ such windows, one for each $n$-gram. Therefore $d=s^n$.
\end{proof}

\begin{proposition}[Extremal project--lift KL gap]\label{prop:extremal-projectlift-gap}
Fix an alphabet $\Sigma$ of size $s \ge 2$, integers $r > n \ge 1$, and any corpus $r$-gram distribution
$\rho$ on $\Sigma^r$. Let
\[
\widetilde{\rho} := G_r(R_n(\rho))
\]
be the rollout from the induced order-$n$ continuation law. Then
\begin{equation}\label{eq:projectlift-extremal-bound}
\KL(\rho \,\|\, \widetilde{\rho})
=
\sum_{t=n+1}^{r}
I_\rho\!\bigl(X_t;\,X_1^{\,t-n}\mid X_{t-n+1}^{\,t-1}\bigr)
\;\le\;
(r-n)\log_2 s
\quad\text{bits}.
\end{equation}
Equivalently, if $r=n+L$, then
\[
\KL(\rho \,\|\, \widetilde{\rho}) \le L \log_2 s \quad\text{bits}.
\]

The bound is sharp. Let $\omega$ be a cyclic de~Bruijn sequence of order $n$ over $\Sigma$, equivalently an
Eulerian cycle in the de~Bruijn graph $B(n-1,s)$. Let $\rho_\omega^{(r)}$ be the uniform distribution on the
$s^n$ cyclic length-$r$ windows cut from $\omega$. Then
\[
\KL\!\Bigl(\rho_\omega^{(r)} \,\Big\|\, G_r(R_n(\rho_\omega^{(r)}))\Bigr)
=
(r-n)\log_2 s
\quad\text{bits}.
\]
\end{proposition}

\begin{proof}
Let $X_1^r=(X_1,\dots,X_r)\sim \rho$, and write
\[
\widetilde{\rho} := G_r(R_n(\rho)).
\]
By construction, $\widetilde{\rho}$ has the same $n$-gram marginal as $\rho$, and for each $t>n$ it
generates $X_t$ from the induced order-$n$ continuation law extracted from $\rho$. Thus
\[
\widetilde{\rho}(x_1^r)
=
\rho(x_1^n)\prod_{t=n+1}^{r}\rho(x_t\mid x_{t-n+1}^{t-1}),
\]
where the conditional on the right is read from the $n$-gram marginal of $\rho$.

Applying the chain rule for KL divergence gives
\begin{align}
\KL(\rho \,\|\, \widetilde{\rho})
&=
\sum_{t=n+1}^{r}
\mathbb E_\rho\!\left[
\log_2 \frac{\rho(X_t\mid X_1^{t-1})}{\rho(X_t\mid X_{t-n+1}^{t-1})}
\right] \notag\\
&=
\sum_{t=n+1}^{r}
I_\rho\!\bigl(X_t;\,X_1^{\,t-n}\mid X_{t-n+1}^{\,t-1}\bigr).
\label{eq:projectlift-cmi}
\end{align}
Each term is a conditional mutual information, so
\[
I_\rho\!\bigl(X_t;\,X_1^{\,t-n}\mid X_{t-n+1}^{\,t-1}\bigr)
\le
H_\rho(X_t\mid X_{t-n+1}^{\,t-1})
\le
\log_2 s,
\]
because $X_t$ takes values in an alphabet of size $s$. Summing over the $r-n$ indices $t=n+1,\dots,r$
proves~\eqref{eq:projectlift-extremal-bound}.

To show sharpness, let $\omega$ be a cyclic de~Bruijn word of order $n$, which exists by Lemma~\ref{lem:debruijn-euler}. Every $n$-gram
appears exactly once in one period of length $s^n$. Therefore the $n$-gram marginal of the cyclic window
distribution $\rho_\omega^{(r)}$ is uniform:
\[
\rho_\omega^{(n)}(u)=s^{-n}\qquad (u\in\Sigma^n).
\]
Hence every $(n-1)$-context is followed equally often by each token, so the induced order-$n$ continuation
law is uniform:
\[
p_\omega(a\mid c)=\frac{1}{s}
\qquad
(c\in\Sigma^{n-1},\ a\in\Sigma).
\]
It follows that the project--lift rollout is the uniform law on all $r$-grams:
\[
G_r(R_n(\rho_\omega^{(r)}))(x_1^r)=s^{-r}
\qquad
(x_1^r\in\Sigma^r).
\]

On the other hand, $\rho_\omega^{(r)}$ is supported on exactly $s^n$ distinct cyclic length-$r$ windows,
one from each starting position of the de~Bruijn cycle. These windows are distinct because their first $n$
symbols are distinct. Each therefore has mass $s^{-n}$. So
\[
\KL\!\Bigl(\rho_\omega^{(r)} \,\Big\|\, G_r(R_n(\rho_\omega^{(r)}))\Bigr)
=
\sum_{x\in \operatorname{supp}(\rho_\omega^{(r)})}
s^{-n}\log_2\frac{s^{-n}}{s^{-r}}
=
(r-n)\log_2 s.
\]
This proves sharpness.
\end{proof}

\begin{remark}[Significance]
Proposition~\ref{prop:extremal-projectlift-gap} identifies the absolute scale of the project--lift diagnostic.
No normative publication rule with hidden extra depth $L=r-n$ can produce a larger block-level KL gap than
$L\log_2 s$ bits, because every such rule ultimately produces some corpus $r$-gram distribution $\rho$, and
the bound above is universal over all such $\rho$.

The de~Bruijn construction shows that this worst case is not merely formal but attainable. Each hidden step
beyond the order-$n$ window can contribute a full $\log_2 s$ bits of irreducible mismatch. This is not a
literal model of natural language, but it is a sharp combinatorial benchmark: a positive project--lift KL gap
need not be a small perturbative effect; in principle it can scale linearly with the hidden extra depth preserved
by normative selection.
\end{remark}

\begin{remark}[The scale in the matched experiment]
In the matched experiment of Section~5.3, we have $|\Sigma|=5$, $n=3$, and $r=5$, so the universal upper
bound is
\[
(r-n)\log_2 |\Sigma|
=
2\log_2 5
\approx 4.64 \text{ bits.}
\]
The observed normative plateau of $2.57$ bits therefore attains about $55\%$ of the maximum possible
project--lift KL gap at this depth.
\end{remark}

\begin{remark}[Companion $L^1$ gap]
For the same extremal de~Bruijn construction,
\[
\bigl\|\rho_\omega^{(r)} - G_r(R_n(\rho_\omega^{(r)}))\bigr\|_1
=
2\bigl(1-s^{\,n-r}\bigr)
=
2\bigl(1-s^{-L}\bigr),
\]
so the $L^1$ diagnostic is also asymptotically maximal as the hidden extra depth $L$ grows.
\end{remark}

\subsection{Section 5 in brief: completing Theorem~2}\label{sec:section5-brief}

This section has developed the project--lift diagnostics---the $L^1$ distance and KL divergence between the corpus $r$-gram distribution and the rollout from its induced order-$n$ continuation law---and used them to complete the quantitative content of Theorem~\ref{thm:tut3-fixedpoint}, Part~(b).

\emph{Strict positivity.}
When the normative quality standard demands structure that no order-$n$ continuation law can produce, the fixed-point corpus $r$-gram distribution $\rho^\star$ is not $n$-shallow, so $\KL(\rho^\star \,\|\, G_r(R_n(\rho^\star))) > 0$.

\emph{Optimal upper bound.}
Proposition~\ref{prop:extremal-projectlift-gap} establishes that the project--lift KL gap satisfies
\[
\KL(\rho \,\|\, G_r(R_n(\rho))) \;\le\; L\log_2 s \quad\text{bits}
\]
for any corpus $r$-gram distribution $\rho$, where $L = r - n$ is the hidden extra depth and $s = |\Sigma|$ is the alphabet size.
The bound is attained by the uniform distribution on cyclic windows from a de~Bruijn word of order~$n$ (Lemma~\ref{lem:debruijn-euler}), so it is optimal.

Together, these results complete Part~(b) of Theorem~\ref{thm:tut3-fixedpoint}: normative publication with hidden extra depth $L$ produces a strictly positive KL gap that is bounded above by the optimal value $L\log_2 s$ bits.

\newpage
\section{Extensions: heterogeneous populations and the argmax limit}\label{sec:heterogeneous}

\subsection{Orientation}

The preceding sections studied populations in which all agents share the same model order and generate at the same temperature. This section records two extensions. First, public text environments are typically heterogeneous: some agents use shorter contexts and others longer ones; some sample at high temperature and others publish near-deterministically. Second, one boundary case of this heterogeneity---the argmax limit $T=0$---admits an exact fixed-point characterisation.

This section is exploratory and is not needed for understanding Theorems~1--3 on a first reading. Its purpose is to show how the core framework extends beyond a single homogeneous population, and to isolate one exactly solvable boundary case where the geometry changes sharply.

\subsection{Setup: a population of $(n_i,T_i)$-agents}

Consider a population of $K$ agent classes, indexed by $i=1,\dots,K$. Each class is specified by:
\begin{itemize}[leftmargin=*]
\item a model order $n_i$;
\item a publication temperature $T_i>0$; and
\item a population share $\omega_i>0$ with $\sum_i \omega_i=1$.
\end{itemize}

Fix a common output block length
\[
r \;\ge\; \max_i n_i .
\]
The public corpus at generation $t$ is represented by an $r$-gram distribution $\rho_t$. From $\rho_t$, class $i$ extracts its own order-$n_i$ continuation law $R_{n_i}(\rho_t)$, samples from it at temperature $T_i$, and contributes an $r$-gram distribution back to the corpus.

At temperature $T$, the continuation law induced by $p$ is
\[
p_T(a\mid c)
=
\frac{p(a\mid c)^{1/T}}{\sum_b p(b\mid c)^{1/T}}.
\]
At $T=1$ this recovers the unmodified law; as $T\to\infty$ it approaches the uniform distribution; as $T\downarrow 0$ it approaches deterministic argmax publication.

\subsection{The mixed-population recursion}

With the common output block length $r$ fixed, the aggregate recursion is
\[
\rho_{t+1}
=
\sum_{i=1}^{K}
\omega_i\, G_r^{(n_i,T_i)}\!\bigl(R_{n_i}(\rho_t)\bigr),
\]
where $G_r^{(n_i,T_i)}$ denotes: extract the order-$n_i$ continuation law, apply temperature $T_i$, and roll it forward to an $r$-gram distribution.

A fixed point $\rho^\star$ is therefore a corpus distribution that reproduces itself under the weighted mixture of all agent classes. When all $(n_i,T_i)$ are equal, this reduces to the single-population recursion studied earlier.

\subsection{Basic observations}

\paragraph{Existence.}
The simplex of $r$-gram distributions is compact and convex, and the mixed-population operator above is continuous for fixed positive temperatures. Hence at least one fixed point exists by Brouwer's theorem.

\paragraph{Temperature heterogeneity breaks the circulation-polytope structure.}
At $T=1$ the descriptive single-population fixed points form the circulation polytope of Section~\ref{sec:worked-examples}. Once different temperatures are mixed, the update is nonlinear even before any normative filtering is introduced. The resulting fixed-point set need not be convex, even when all agents use the same model order.

\paragraph{Context-window heterogeneity imposes mixed-resolution constraints.}
When some agents use order $n_1$ and others use order $n_2>n_1$, the fixed point is constrained simultaneously by both resolutions. It need not be individually self-consistent at either resolution alone, but it must be stable under their weighted combination.

\paragraph{A tractable boundary case.}
A full characterisation of the general heterogeneous recursion is open. The one boundary regime that can be analysed exactly is the deterministic argmax limit $T=0$, to which we now turn.

\subsection{The exactly solvable boundary case: homogeneous argmax agents}\label{sec:T0-fixedpoints}

To isolate the effect of deterministic publication, fix a common model order $n$ and take the limit $T=0$. We assume a fixed deterministic tie-breaking rule whenever the argmax is not unique. Thus every visited context $c$ has a single published successor, denoted $f(c)$.

Individual texts are then deterministic, but the population still produces a nontrivial $n$-gram distribution because the public corpus is a collection of independently initiated texts. Different initial contexts can flow into different periodic orbits of the deterministic map $f$, so the population-level corpus may remain a genuine mixture.

\paragraph{Why vertex-disjointness is required.}
At $T=1$, convex combinations of cycle distributions are self-consistent because the mixture conditional reproduces both branches with the correct frequencies. At $T=0$, this fails whenever two cycles visit the same context with different successors: the argmax chooses only one successor, destroying the minority branch. By contrast, if two cycles are vertex-disjoint, then every visited context has a unique successor with conditional probability $1$, so the argmax is automatically self-consistent regardless of the mixture weights.

\paragraph{Worked example: $s=3$, $n=3$.}
The de~Bruijn graph $B(2,3)$ has $9$ nodes (the bigram contexts) and $27$ edges (the trigrams). Consider the two vertex-disjoint cycles
\[
\gamma_A: 00 \to 00
\qquad\text{and}\qquad
\gamma_B: 01 \to 12 \to 20 \to 01.
\]
They correspond to the periodic texts
\[
\ldots 000\ldots
\qquad\text{and}\qquad
\ldots 012012\ldots
\]
and their context sets $\{00\}$ and $\{01,12,20\}$ are disjoint. Hence for any $\lambda\in(0,1)$ the mixture
\[
\rho=\lambda\,\rho_{\gamma_A}+(1-\lambda)\,\rho_{\gamma_B}
\]
is a $T=0$ fixed point.

A larger example is obtained from the six pairwise vertex-disjoint cycles
\[
\underbrace{\ldots 000\ldots,\ \ldots 111\ldots,\ \ldots 222\ldots}_{\text{three self-loops}}
\quad\cup\quad
\underbrace{\ldots 0101\ldots,\ \ldots 0202\ldots,\ \ldots 1212\ldots}_{\text{three period-$2$ cycles}},
\]
which cover all nine contexts. Their convex hull is a $5$-simplex of $T=0$ fixed points. For reference, exhaustive enumeration in $B(2,3)$ yields $843$ vertex-disjoint cycle collections in total, of which $216$ are maximal.

\begin{proposition}[$T{=}0$ fixed points]\label{prop:T0-fp}
Fix $s\ge 2$ and $n\ge 2$, and consider the infinite-corpus descriptive recursion at temperature $T=0$ with a fixed deterministic tie-breaking rule. A distribution $\rho$ on $n$-grams is a $T=0$ fixed point if and only if
\[
\rho
=
\sum_{k=1}^{m}\lambda_k\,\rho_{\gamma_k},
\qquad
\lambda_k>0,\qquad
\sum_{k=1}^{m}\lambda_k=1,
\]
where $\gamma_1,\dots,\gamma_m$ are pairwise vertex-disjoint simple directed cycles in the de~Bruijn graph $B(n-1,s)$ and $\rho_{\gamma_k}$ is the uniform distribution on the $n$-grams traversed by $\gamma_k$.
\end{proposition}

\begin{proof}
Let $f$ be the deterministic successor map induced by argmax and the fixed tie-breaking rule. A corpus distribution $\rho$ is a fixed point exactly when it is stationary under the deterministic update $f$.

If $\rho$ is supported on pairwise vertex-disjoint cycles $\gamma_1,\dots,\gamma_m$, then every visited context has a unique successor with conditional probability $1$, so the argmax map agrees with the successor on each cycle. The uniform distribution on each cycle is stationary under $f$, and any convex combination of these cycle measures is therefore stationary. This proves sufficiency.

Conversely, let $\rho$ be a $T=0$ fixed point. Because $f$ is deterministic on a finite state space, every context eventually enters a periodic orbit. A stationary distribution cannot place positive mass on transient contexts, since transient mass is shifted away after one step. Hence the support of $\rho$ lies entirely on periodic orbits of $f$, i.e.\ on simple directed cycles in $B(n-1,s)$. Distinct cycles of a deterministic map are automatically vertex-disjoint, and the stationary distribution restricted to each cycle is uniform. Therefore $\rho$ is a convex combination of uniform cycle distributions on pairwise vertex-disjoint simple cycles.
\end{proof}

\subsection{Geometry of the $T{=}0$ fixed-point set}

The $T=0$ fixed-point set is a union of simplices:
\[
\mathcal F_n^{T=0}
=
\bigcup_{\{\gamma_1,\dots,\gamma_m\}}
\mathrm{conv}\{\rho_{\gamma_1},\dots,\rho_{\gamma_m}\},
\]
where the union runs over all collections of pairwise vertex-disjoint simple directed cycles. Each collection of $m$ cycles contributes an $(m-1)$-simplex of mixture weights.

This contrasts sharply with the $T=1$ case, where the full circulation polytope is convex and all cycle distributions can be mixed freely. At $T=0$, only vertex-disjoint cycles can coexist, so the geometry becomes topologically richer but globally non-convex.

Marginalisation remains linear inside any fixed simplex: if
\[
\rho=\sum_k \lambda_k \rho_{\gamma_k},
\]
then every shorter $k$-gram marginal is
\[
\rho^{(k)}=\sum_k \lambda_k (\rho_{\gamma_k})^{(k)}.
\]
What fails globally is convexity across incompatible simplices: mixtures of fixed points supported on cycle families that share a context are not themselves fixed points, because the shared context forces a single argmax successor.

\subsection{Section 6 in brief}

Heterogeneous populations extend the core framework by allowing different context windows and temperatures. In full generality the resulting fixed-point geometry is nonlinear and remains open. The exactly solvable boundary case is the argmax limit $T=0$: there the fixed-point set is a union of simplices generated by vertex-disjoint cycle distributions. These are exploratory extensions of the core theory rather than part of the main theorem chain.

\subsection{Directions}

A full characterisation of the mixed $(n_i,T_i)$ recursion remains open. Natural next questions include uniqueness versus multiplicity of fixed points under temperature heterogeneity, the geometry induced by mixtures of different context windows, and the interaction between deterministic argmax components and positive-temperature exploratory components.

\newpage
\section{Experiment matrix overview}\label{sec:experiment-matrix}

\begin{table}[h]
\centering
\scriptsize
\setlength{\tabcolsep}{3pt}
\caption{Matrix overview for appendix experiments and protocols.}
\label{tab:appx-experiment-matrix}
\begin{tabular}{p{0.10\linewidth}p{0.07\linewidth}p{0.13\linewidth}p{0.23\linewidth}p{0.17\linewidth}p{0.12\linewidth}}
\toprule
Theorem & Sign & Corpus/task & What is recursively removed or filtered & Main observable & Section \\
\midrule
Theorem~1 & Negative & Doyle, Austen, Darwin & Rare lexical tail and high-order support under fixed-size recursion & Vocabulary and trigram retention & Section~\ref{sec:ngram-basics} \\
Theorem~1 & Positive & Austen orthographic pilot & Low-support orthographic variants under fixed-size recursion & Active variants, canonical share & Section~\ref{sec:ngram-basics} \\
Theorem~2 & Fixed point & Synthetic exact runs & Soft vs.\ hard publication diagnostics; descriptive vs.\ normative KL split & KL divergence, $\ell_1$ convergence, entropy & Sections~\ref{sec:selection-theory}--\ref{sec:info-diagnostics} \\
Theorem~3 & Both & Synthetic + Doyle & Environment reshaped by drift and acceptance filtering & Target conditional recovery by cross-entropy learners & Sections~\ref{sec:selection-theory} \\
\bottomrule
\end{tabular}
\end{table}

\section*{Data and code availability}

The exact code used to generate all figures and tables is available at \url{\repoURL}.
Text corpora are public-domain Project Gutenberg editions.
The main notebooks and helper modules are in the repository's \texttt{notebooks/} and \texttt{src/} directories.

\end{document}